\documentclass{article}

\usepackage[preprint]{neurips_2026}
\usepackage[utf8]{inputenc} % allow utf-8 input
\usepackage[T1]{fontenc}    % use 8-bit T1 fonts
\usepackage{hyperref}       % hyperlinks
\usepackage{url}            % simple URL typesetting
\usepackage{booktabs}       % professional-quality tables
\usepackage{amsfonts}       % blackboard math symbols
\usepackage{nicefrac}       % compact symbols for 1/2, etc.
\usepackage{microtype}      % microtypography
\usepackage{xcolor}         % colors
\usepackage{amsmath,amssymb,amsthm,mathrsfs}
\usepackage[utf8]{inputenc}   % allow utf-8 input
\usepackage[T1]{fontenc}      % use 8-bit T1 fonts
\usepackage{hyperref}         % hyperlinks
\usepackage{url}              % simple URL typesetting
\usepackage{booktabs}         % professional-quality tables
\usepackage{amsfonts}         % blackboard math symbols
\usepackage{nicefrac}         % compact symbols for 1/2, etc.
\usepackage{microtype}        % microtypography
\usepackage{xcolor}           % colors
\usepackage{graphicx}         % figures
\usepackage{amsthm}
\usepackage{amsmath}   % 如果还没加，建议一起加
\graphicspath{{figures/sources/}}

\usepackage{CJKutf8}
\usepackage{enumitem}
\usepackage{float}

\newtheorem{theorem}{Theorem}
\newtheorem{proposition}[theorem]{Proposition}

\theoremstyle{definition}
\newtheorem{definition}{Definition}
\newtheorem{example}{Example}
\newtheorem{remark}{Remark}

\title{Demystifying Oversmoothing in Sheaf Neural Networks: An Index-Theoretic Criterion}

\author{%
  Junwen Dong\thanks{Junwen Dong, Yuhan Peng and Hao Li contributed equally to this work.} \\
  Chern Institute of Mathematics \\
  Nankai University \\
  Tianjin, China \\
  \And
  Yuhan Peng\footnotemark[1] \\
  School of Physical and Mathematical Sciences \\
  Nanyang Technological University \\
  Singapore \\
  \AND
  Hao Li\footnotemark[1] \\
  Department of Mathematics, Faculty of Science \\
  National University of Singapore \\
  Singapore \\
  \And
  Huitao Feng \\
  Chern Institute of Mathematics \\
  Nankai University \\
  Tianjin, China \\
  \AND
  Kelin Xia\thanks{Corresponding author.} \\
  School of Physical and Mathematical Sciences \\
  Nanyang Technological University \\
  Singapore \\
  \texttt{xiakelin@ntu.edu.sg}
}

\begin{document}
\maketitle

\begin{abstract}
  % Oversmoothing in graph neural networks occurs when repeated message passing collapses node representations onto a low-dimensional harmonic subspace. Sheaf neural networks counter this by enriching the diffusion operator with stalk spaces and restriction maps, but a large harmonic dimension alone does not guarantee anti-oversmoothing—trivial sheaves can produce high-dimensional harmonic spaces consisting entirely of constant modes. We derive an index-theoretic comparison criterion that certifies when one sheaf's harmonic space strictly contains another's as a genuine geometric inclusion, combining a holonomy condition with a cohomological compatibility requirement. The framework unifies standard GNNs, Euclidean and manifold-valued sheaf models under a single invariant. We introduce GyroSheaf, a sheaf network with intrinsically curved stalks, and prove it admits a tangent Hodge decomposition via linearization. Experiments across ten models on three datasets confirm the predicted four-tier hierarchy: sheaf models violating the criterion collapse despite possessing index jumps, while compliant models maintain depth-stable representations.

To combat oversmoothing in Graph Convolutional Networks, Sheaf Neural Networks (SNNs) were proposed as a generalization by equipping the graph with a sheaf structure and replacing the graph Laplacian with a sheaf Laplacian $\mathcal{L}$. Existing analyses connect sheaf diffusion to oversmoothing via the harmonic space ($\ker\mathcal{L}$), taking its absolute dimension as an indicator of anti-oversmoothing capacity. However, absolute dimension alone is not a reliable measure: certain sheaf configurations inflate $\dim \ker \mathcal{L}$ while their harmonic sections remain entirely constant, without enriching discriminative capacity. We instead introduce the first relative, geometric approach, yielding a precise characterisation of anti-oversmoothing capacity. Under natural conditions on stalk transportation and global sheaf structure, we establish an index-theoretic comparison criterion showing that one sheaf's harmonic space genuinely contains another's beyond trivial inflation. We illustrate this  with a concrete instance
%with SPDsheaf as a concrete instance 
and further introduce \textit{GyroSheaf}, a sheaf with curved gyrovector-space stalks, extending the criterion to the non-linear setting via local tangent-space linearization. Experiments across ten models confirm the theoretical criterion: sheaf models violating the criterion collapse despite possessing index jumps, while compliant models maintain depth-stable representations.

\end{abstract}

% ---- Body ----
\section{Introduction}

Oversmoothing is the phenomenon where node representations  become increasingly similar under repeated message passing, reducing their ability to separate node classes~\cite{li2018deeper}. In the idealized diffusion view, this corresponds to the asymptotic projection of node signals onto the harmonic space of the graph Laplacian. For a connected graph, harmonic space consists of constant
signals, so the diffusion limit collapses all node features onto a single global mode, precisely the classical oversmoothing phenomenon.  Sheaf Neural Networks (SNNs)~\cite{SNN2020} generalize graph diffusion by equipping a graph with stalk spaces and restriction maps (setting all stalks to $\mathbb{R}$ with identity maps recovers the graph Laplacian of GCN~\cite{bodnar2022nsd}), yielding a sheaf Laplacian whose harmonic space can retain nontrivial sections that preserve node-level differences even in the deep limit. Bodnar et al.~\cite{bodnar2022nsd} further studied the relationship between the harmonic space of the sheaf Laplacian and oversmoothing in SNNs. These results highlight the central role of the harmonic space in understanding oversmoothing. However, for a fixed sheaf, the dimension of the harmonic space alone does not provide a universal measure of anti-oversmoothing capacity. Certain sheaves may possess a large harmonic space whose sections are nonetheless constant; conversely, a small
harmonic space may consist entirely of nonconstant, class-preserving
sections. Thus
anti-oversmoothing cannot be characterised merely by the rank or dimension of
the harmonic space of a single sheaf. This motivates a relative, geometric comparison: rather than measuring the harmonic space of a single sheaf in isolation, we ask what additional harmonic structure a target sheaf $\mathcal{F}$ admits beyond that of a reference sheaf
$\xi$.  

\textbf{Our approach.} A first natural indicator of this relative capacity is the index jump $\Delta_{ind}(\mathcal{F},\xi):=ind(D_{\mathcal{F}}^{+})-ind(D_{\xi}^{+})$ which yields the harmonic-capacity difference
$\dim H^{0}(\mathcal{F})-\dim H^{0}(\xi)$
(Proposition~\ref{prop:index_jump}). However, a positive raw index jump alone does not guarantee genuine anti-oversmoothing: trivial configurations (e.g., IdentitySheaf or InverseSheaf; see Example~\ref{example}) can inflate $\dim H^{0}$ via channel-wise constant sections that still collapse. Thus, the raw index cannot distinguish meaningful harmonic enlargement from trivial stalk dimension inflation.

This led us to propose a more detailed and structurally informed set of criterion. Our main result (Theorem ~\ref{thm:main}) establishes that the index jump must be combined with a degree-one heat trace correction $\Delta_{Str}^{(1)}$ to properly control $H^{0}$. Specifically, when the restriction maps induce non-trivial holonomy (preventing full channel decoupling), and the corrected index balance satisfies our capacity conditions, we prove that $H^{0}(\mathcal{F})/\iota_{*}H^{0}(\xi)\ne0$. Rather than a mere numerical comparison, this establishes a genuine geometric strict inclusion: the harmonic space of $\mathcal{F}$ strictly contains that of $\xi$ embedded via $\iota_{\ast}$. This \emph{excess harmonic capacity} precisely measures $\mathcal{F}$'s structural resistance to oversmoothing. Furthermore, Section~\ref{sec:nonlinear} extends this criterion to sheaves with non-linear stalks.

\textbf{Contributions.} Our main contributions are threefold: (1) \textbf{A precise mathematical characterisation for anti-oversmoothing capacity.}
  We establish (Theorem~\ref{thm:main}) the first algebraic-topological
  inclusion criterion that enables a rigorous comparison of anti-oversmoothing capacity between arbitrary sheaf frameworks, validated by oversmoothing experiments across ten models on three datasets. (2) \textbf{Extension to non-linear stalk structures.}
  We generalise the index-theoretic framework beyond the Euclidean setting
  to sheaves whose stalks carry non-linear geometry (e.g.\ SPD manifolds),
  broadening the scope of the theory to sheaf architectures that operate on
  curved feature spaces. (3) \textbf{GyroSheaf: a concrete non-linear instantiation.}
  We construct GyroSheaf, a sheaf model whose stalks carry
gyrovector-space structure, and prove that it admits a tangent Hodge
decomposition. This construction demonstrates the design principles of
our criterion on curved stalks. We also provide a direct comparison between the
nonlinear and linear sheaf frameworks on curved stalks.

\section{Related Work}

\textbf{Oversmoothing in GNNs.}
The oversmoothing behavior of message-passing GNNs—GCN~\cite{kipf2017gcn}, GAT~\cite{velickovic2018gat}, GraphSAGE~\cite{hamilton2017graphsage}—has been 
characterized both as Laplacian smoothing~\cite{li2018deeper} and as exponential convergence to a low-dimensional subspace determined by the graph's connected components and node 
degrees~\cite{oono2020graph}; see~\cite{rusch2023survey} for a survey. %Existing remedies are largely regularization-based: PairNorm~\cite{zhao2020pairnorm}, DropEdge~\cite{rong2020dropedge}, GCNII~\cite{chen2020gcnii}, and residual or energy-constrained variants, which mitigate the symptom without altering the underlying operator geometry that causes it.
Existing remedies fall into three broad families:  \emph{normalization/regularization} approaches that suppress feature collapse layer-wise~\cite{zhao2020pairnorm, zhou2021dirichlet}; 
%such as PairNorm~\cite{zhao2020pairnorm} and DropEdge~\cite{rong2020dropedge} suppress feature collapse layer-wise. 
\emph{residual/skip-connection} variants that preserve earlier layers' signals~\cite{xu2018jknet, chen2020gcnii}; 
% rusch2023gradient
% : JKNet~\cite{xu2018jknet}, 
% GCNII~\cite{chen2020gcnii}, and gradient gating 
% ($\mathrm{G}^2$)~\cite{rusch2023gradient}, preserve earlier layers' signals through architectural shortcuts. 
\emph{energy or flow-based} methods that recast propagation as a continuous dynamical system and control its limit~\cite{chamberlain2021grand, digiovanni2023understanding}. 
% : GRAND~\cite{chamberlain2021grand} treats GNNs as 
% diffusion PDEs, EGNN~\cite{zhou2021dirichlet} imposes upper and lower Dirichlet-energy constraints at each layer, and 
% GRAFF~\cite{digiovanni2023understanding} formulates propagation as the gradient flow of a learnable energy with both attractive and repulsive modes. All of these mitigate oversmoothing through architectural or regularization mechanisms, rather than by structurally enlarging the space of signals that the diffusion operator preserves.
These methods do not, however, directly change the harmonic subspace that determines the infinite-depth limit of the underlying graph diffusion. This motivates a complementary line of work that changes the diffusion operator itself; sheaf Laplacians~\cite{SNN2020, bodnar2022nsd, hansen2019toward} provide such a mechanism by making the limiting harmonic space depend on stalks and restriction maps rather than only on graph connectivity.

\textbf{Sheaf neural networks and geometric extensions.}
Sheaf neural networks~\cite{SNN2020,bodnar2022nsd} replace graph diffusion with sheaf diffusion, enabling a richer harmonic space that can preserve non-constant signals at the diffusion limit, with subsequent extensions to hypergraphs and cell complexes~\cite{hyper, barbero2022sheaf,battiloro2023generalized}. However, existing sheaf-theoretic analyses characterize anti-oversmoothing capacity through the absolute dimension $\dim H^0$ of a single sheaf, without a \emph{relative} criterion for comparing two sheaf frameworks. %Caralt et al.~\cite{recent_critique_2026} expose this gap empirically by showing that identity sheaves match learned variants at moderate training depth.
On the geometric side, hyperbolic GNNs~\cite{liu2019hyperbolic} and SPD-valued GNNs~\cite{lopez2021modeling,nguyen2023enhancing} perform message passing in non-Euclidean spaces, with~\cite{our} further extend SPD geometry to the sheaf setting.
Gyrovector algebras~\cite{ungar2008gyrovector} have been adopted in hyperbolic~\cite{chen2022deephgcn} and SPD~\cite{lopez2021vector, Nguyen22} graph networks, but have not been used to construct sheaf coboundary operators.
Neither line of work provides a relative criterion for strict harmonic inclusion—our index comparison theorem fills this gap, and GyroSheaf extends it to curved gyrovector stalks as a non-linear instantiation.

\section{A Sheaf-Theoretic Diffusion Framework on Graphs}
\label{sec:framework}
% \begin{CJK}{UTF8}{gbsn}{
% \py{M没有定义, 如果需要额外篇幅，看这个Preliminaries的第一段要不要删掉一些}
% \begin{CJK}{UTF8}{gbsn}{M定义好啦：第五排Riemannian manifold M。我觉得第二段写得非常好
% }\end{CJK}
% % \py{很多东西（除了pairing）的部分可能\cite{barbero2022sheaf}还有NSD都定义过了，比如Theorem 2应该在toward a spectral theory of cellular sheaves里面提到，要不要改成Background: A Discrete Sheaf-Hodge Framework on Graphs把这一个section的篇幅缩小（放一部分去appendix）这样我们整体还能少一些内容在正文？def3（我们推广到了一般的sheaf stalk space）, proposition 1（整合）和proposition 2应该是没人做过}
% }\end{CJK}

\textbf{Preliminaries.} 
This section develops the framework of sheaf diffusion on graphs, which forms the basis for the index-theoretic analysis in Section~\ref{sec:main}. Our discrete construction is guided by continuous heat diffusion: on a compact Riemannian manifold $M$, under the positive-semidefinite convention for the Laplace--Beltrami operator $\Delta$, the heat equation $\partial_t f = -\Delta f$ is the $L^2$--gradient flow of the Dirichlet energy 
$E(f)= \frac{1}{2} \int_M \|\nabla f\|^2$. Along this flow, the energy is monotonically dissipated, and the solution converges as $t \to \infty$ to its projection onto the harmonic space $\ker\Delta$. By Hodge theory, harmonic spaces are isomorphic to de Rham cohomology groups, connecting diffusion equilibria with the topology of the underlying space.

We reproduce this structure in the discrete sheaf setting. The core algebraic objects
% ---coboundary operator, adjoint, sheaf Laplacian, and the 0th-order Hodge decomposition---
are well established for vector-space stalks~\cite{SNN2020, bodnar2022nsd, hansen2019toward}. We restate them in a unified form extending to general stalk spaces (e.g., SPD matrices). Definition~\ref{def: def1},~\ref{def: def2} and Theorem~\ref{thm:thm2} recover the classical formulations of~\cite{bodnar2022nsd, hansen2019toward} for linear stalks. Remark~\ref{rmk2} notes that the algebraic definitions requiring addition and scalar multiplication (Definition~\ref{def3}, Proposition~\ref{prop1}) generalize to SPD stalks either via the Lie group operation~\cite{our} required for Section~\ref{sec:main} or via tangent-space linearization required for Section~\ref{sec:nonlinear} when no group is assumed.

\subsection{Sheaves on Graphs}
\label{sec:1}
\begin{definition}[Discrete sheaf] \label{def: def1}
A \emph{discrete sheaf} $\mathcal{F}$ on an undirected graph $G = (V, E)$ consists of:
\begin{itemize}[leftmargin=2em]
    \item A stalk space $\mathcal{F}_v$ to each node $v \in V$, and a stalk space $\mathcal{F}_e$ to each edge $e \in E$;
    \item For every incident pair $v \in e$, a \emph{restriction map} $\mathcal{F}_{v \to e} \colon \mathcal{F}_v \to \mathcal{F}_e$, acting as a discrete parallel transport. The restriction maps encode the compatibility relations between adjacent stalks.
\end{itemize}
In the linear setting, each stalk is a finite-dimensional vector space (e.g.\ $\mathcal{F}_v = \mathbb{R}^{n}$) and each restriction map is a linear transformation. In the nonlinear setting, each stalk is a Riemannian manifold (e.g.\ $\mathcal{F}_v = \mathrm{SPD}_n$, $n\times n$ SPD matrices) and each restriction map is a smooth map between manifolds. The \emph{0-cochain space} is $\mathcal{C}^0(G, \mathcal{F}) := \prod_{v \in V} \mathcal{F}_v$ and the \emph{1-cochain space} is $\mathcal{C}^1(G, \mathcal{F}) := \prod_{e \in E} \mathcal{F}_e$.
\end{definition}
\subsection{Coboundary Operator, Admissible Pairing, and Adjoint Operator}
\label{sec:2}
\begin{definition}[Discrete coboundary]\label{def: def2}
Let $\mathcal F$ be a sheaf on a graph $G$. The discrete coboundary operator $\delta : \mathcal C^0(G,\mathcal F) \to \mathcal C^1(G,\mathcal F)$
% \[
% \delta : \mathcal C^0(G,\mathcal F) \to \mathcal C^1(G,\mathcal F)
% \]
maps node-wise assignments to edge-wise discrepancies. For $X\in \mathcal C^0(G,\mathcal F)$ and an oriented edge $e=(u,v)$, $(\delta X)_e \in \mathcal F_e$ compares the transported endpoint features $\mathcal{F}_{u\to e} X_u$ and $\mathcal{F}_{v\to e} X_v$ in the common edge stalk $\mathcal F_e$.
\end{definition}
\begin{remark}
The concrete form of $\delta$ depends on the geometry of the stalks. For vector-space stalks, it is given by the signed difference:
$(\delta X)_e
=\mathcal{F}_{u\to e} X_u-\mathcal{F}_{v\to e} X_v$.
For general geometric stalks, this subtraction is replaced by an appropriate discrepancy map on $\mathcal F_e$.
\end{remark}

To define the adjoint $\delta^\ast$ of the coboundary operator
\(\delta:\mathcal C^0\to \mathcal C^1\), we need pairings on the cochain spaces. Unlike a smooth manifold equipped with a Riemannian metric, where the $L^2$ pairing is induced by the metric and volume form, a discrete sheaf  lacks a canonical global cochain pairing. We therefore construct the \textit{global cochain pairings} explicitly by summing the local pairings on the stalks: $\langle X, X'\rangle_0 := \sum_{v \in V} \langle X_v, X'_v \rangle_{\mathcal{F}_v}$ for $X, X' \in \mathcal{C}^0$ and $\langle Y, Y'\rangle_1 := \sum_{e \in E} \langle Y_e, Y'_e \rangle_{\mathcal{F}_e}$ for $Y, Y' \in \mathcal{C}^1$.

% \[
% \langle X, X'\rangle_0 := \sum_{v \in V} \langle X_v, X'_v \rangle_{\mathcal{F}_v},
% \qquad
% \langle Y, Y'\rangle_1 := \sum_{e \in E} \langle Y_e, Y'_e \rangle_{\mathcal{F}_e}
% \]
% for $X, X' \in \mathcal{C}^0$ and $Y, Y' \in \mathcal{C}^1$.

\begin{definition}[Admissible pairing]\label{def3}
We say that a global cochain pairing is \emph{admissible} if for $i\in\{0,1\}$, the pairing
$\langle\cdot,\cdot\rangle_i$ is bilinear, symmetric, and positive definite. That is,
for $X,Y,Z \in \mathcal C^i$ and $\alpha,\beta \in \mathbb{R}$,
\[
\langle \alpha X+\beta Y,Z\rangle_i
= \alpha\langle X,Z\rangle_i + \beta\langle Y,Z\rangle_i,
\qquad
\langle X,Z\rangle_i = \langle Z,X\rangle_i,
\]
and $\langle X,X\rangle_i > 0$ for all $X \neq 0$.
\end{definition}
\begin{definition}[Adjoint operator]
With respect to these pairings, the adjoint operator $\delta^\ast$ satisfies: $\langle \delta X,Y\rangle
=
\langle X,\delta^\ast Y\rangle$, for all $X\in \mathcal C^0,\ Y\in \mathcal C^1$.
\end{definition}
\subsection{Sheaf Laplacian and Associated Dirichlet Energy}
\begin{definition}[Sheaf Laplacian and Dirichlet Energy]
\label{def: def5}
The sheaf Laplacian on $\mathcal C^0$ is defined as $\mathcal{L} := \delta^\ast \delta$. and the associated Dirichlet energy is defined as $\mathcal E(X)
:=
\frac12\|\delta X\|^2
$. 
\end{definition}
Analogous to the Laplace--Beltrami operator $\Delta$ on Riemannian 
manifolds, $\mathcal{L}$ induces the sheaf heat equation 
$\dot X(t)=-\mathcal{L}X(t)$, whose solution converges to 
$\ker\mathcal{L}$ as $t\to\infty$. An explicit Euler discretization 
yields the sheaf diffusion layer $X_{k+1}=\sigma((I-\mathcal{L})X_k)$.
And the following proposition shows the energy-Laplacian relationship
mirrors the continuous setting.

\begin{remark}\label{rem:two-paths}
\label{rmk2}
These operators admit two natural nonlinear extensions:
when a Lie-group structure is available, $\delta$ and $\delta^{\ast}$
lift to group operations (Section~\ref{sec:casespd});
when no such structure is assumed, one instead linearises at each
cochain and works in the local tangent space (Section~\ref{sec:nonlinear}).
\end{remark}

\begin{proposition}[Energy-Laplacian correspondence]\label{prop1}
The Laplacian $\mathcal{L} = \delta^\ast \delta$ is tied to the discrete energy $\mathcal{E}(f)$ through the following properties:
\begin{itemize}[leftmargin=2em]
    \item \(L\) is the gradient of the Dirichlet energy:
    $\nabla\mathcal E(X)=\mathcal{L}X$.
    \item The quadratic form of \(\mathcal{L}\) equals the total edge-wise sheaf inconsistency:
   $\langle \mathcal{L}X,X\rangle_0=\|\delta X\|_1^2=2\mathcal E(X)$.
    \item The positive semidefiniteness of \(L\) ensures stability of the continuous diffusion process $\partial_t X_t=-\mathcal{L}X_t:$ it dissipates energy and stabilizes at harmonic equilibria.
\end{itemize}
\end{proposition}
\subsection{Discrete Hodge Theory}
As noted above, discrete sheaf diffusion drives features toward $\ker\mathcal{L}$, so over-smoothing is governed by the structure 
of this limiting subspace. Following~\cite{bodnar2022nsd, hansen2019toward}, 
Proposition~\ref{prop:anti_oversmoothing_principle} characterizes 
$\ker\mathcal{L}$ in the general stalk setting of 
Definition~\ref{def: def1} and formalizes anti-oversmoothing as a 
structural property of $\mathcal{H}^0$.

\begin{theorem}[0th-order Discrete Hodge Decomposition]\label{thm:thm2}
Let $\mathcal{L}=\delta^\ast\delta$ be the sheaf Laplacian. Then 
\[
\ker \mathcal{L} \cong \ker \delta  \cong H^{0}(G,\mathcal{F})
= \{ X \in C^{0}(G,\mathcal{F}) : \delta X = 0 \}.
\]
\end{theorem}
Thus the diffusion equilibrium is canonically isomorphic to the zeroth 
sheaf cohomology, i.e., the space of \textit{global sections} (defined as $\ker \delta$) satisfying strict 
edgewise consistency. Thus over-smoothing reduces to a structural 
question: how large is the harmonic capacity $\dim\mathcal{H}^0$ of the underlying sheaf?

\begin{proposition}[Harmonic Characterisation of Over-Smoothing; proof in Appendix~\ref{app:anti_oversmoothing_details}]
\label{prop:anti_oversmoothing_principle}
Let $P_{\mathcal{H}^0}:\mathcal{C}^0\to\mathcal{H}^0$ be the projection 
onto harmonic space $\mathcal{H}^0=\ker\mathcal{L}$. For any initial feature assignment 
$f^{(0)}$, the deep limit of sheaf diffusion satisfies 
$f^{(\infty)}=P_{\mathcal{H}^0}f^{(0)}$. The resistance of $\mathcal{F}$ 
to over-smoothing is governed by the representational capacity of 
$\mathcal{H}^0\cong H^0(G,\mathcal{F})$:

(1) \textbf{Flat collapse:} If $\dim(\mathcal{H}^0)$ is minimal $\mathcal{H}^0$ lacks the representational capacity to preserve 
    node-level distinctions, and $f^{(\infty)}$ collapses to a uniform 
    state.\\
(2) \textbf{Geometric preservation:} A richer $\mathcal{H}^0$  provides sufficient representational capacity so that $f^{(\infty)}$ preserves node-level distinctions, preventing collapse. 
% Proof in Appendix~\ref{app:anti_oversmoothing_details}.
\end{proposition}

\section{An Index-Theoretic Analysis of Over-Smoothing}
\label{sec:main}
As established previously, mitigating operator-induced over-smoothing requires a structural enlargement of the zeroth harmonic space. A natural analytic invariant is the "raw" index, but this quantity is Euler-balanced: it records the difference between 0th and 1st discrete sheaf cohomology, not the 0th harmonic capacity alone. 
%Instead of the raw index jump, we choose the interpretation for anti-over-smoothing capacity which combines the index jump with a degree-one heat-supertrace correction. 
We instead combine the index jump with a degree-one heat-supertrace correction and state the resulting criterion in the main text and defer the full derivation to Appendix~\ref{app:C}, including the holonomy-fixed sufficient condition.

\subsection{Index jump}\label{subsec:index_improvement_framework}

We now connect the cohomological objects to the dynamics of deep diffusion. The continuous-depth forward propagation obeys $\partial_t f=-\mathcal{L}_0f$, with solution $f(t)=e^{-t\mathcal{L}_0}f(0)$, where $\mathcal{L}_0=\delta_{\mathcal F}^\ast\delta_{\mathcal F}$ is the $0$-cochain sheaf Laplacian. As $\mathcal{L}_0$ is non-negative and self-adjoint, the heat operator satisfies $\lim_{t\to\infty}e^{-t\mathcal{L}_0}\to P_{\ker \mathcal{L}_0}$, where $P_{\ker \mathcal{L}_0}$ is the projection onto the harmonic space. The topologically stable quantity is the heat supertrace of the Hodge--Dirac pair, which is related to the analysis index of Dirac operator $D_{\mathcal F}$ (see Appendix~\ref{app:C.1} for its full construction.) 
Let $D_{\mathcal F}=\begin{pmatrix} 0 & \delta_{\mathcal F}^\ast\\ \delta_{\mathcal F} & 0 \end{pmatrix}$ denote the Dirac operator on $\mathcal C^0(G;\mathcal F)\oplus \mathcal C^1(G;\mathcal F)$ , with chiral components
\[
D_{\mathcal F}^{+}=\delta_{\mathcal F}:\mathcal C^0(G;\mathcal F)\to \mathcal C^1(G;\mathcal F),\qquad
D_{\mathcal F}^{-}=\delta_{\mathcal F}^\ast:\mathcal C^1(G;\mathcal F)\to \mathcal C^0(G;\mathcal F).
\]

The \textit{"raw" analytic index} is defined by
\[
\operatorname{ind}(D_{\mathcal F}^{+})
:=
\dim\ker D_{\mathcal F}^{+}
-
\dim\ker D_{\mathcal F}^{-}=\dim\ker\delta_{\mathcal F}
-
\dim\ker\delta_{\mathcal F}^{\ast}.
\]
By the discrete Hodge decomposition, this becomes $\operatorname{ind}(D_{\mathcal F}^{+})=\dim H^0(\mathcal F)-\dim H^1(\mathcal F)$.

\begin{proposition}[\textbf{Topological Index and Index Jump}; proof in Appendix~\ref{proof.4}]\label{prop:index_jump}
Let $\mathcal F$ be a rank-$r$ sheaf on a finite graph $G=(V,E)$. Then the topological index is:
\[
\operatorname{ind}(D_{\mathcal F}^{+})
=
\dim \mathcal{C}^0(\mathcal F)-\dim \mathcal{C}^1(\mathcal F)
=
r\chi(G).
\]
depending only on the rank and the Euler characteristic $\chi(G)$. In particular, if $\xi$ is the Euclidean sheaf of rank $n$, $\operatorname{ind}(D_\xi^{+})=n\chi(G).
$

Consequently, The \textbf{index jump} $\operatorname{ind}(D_{\mathcal F}^{+})-\operatorname{ind}(D_{\xi}^{+})$ between the two sheaves then gives the harmonic-capacity difference: $\dim H^0(\mathcal F)-\dim H^0(\xi)
=
\dim H^1(\mathcal F)-\dim H^1(\xi)
+
(r-n)\chi(G)$.
\end{proposition}
% (See detalied proof in Appendix~\ref{proof.4})\textcolor{blue}{
%Crucially, Proposition~\ref{prop:index_jump} shows that the gain in harmonic capacity splits into:$\Delta h^0(\mathcal F,\xi)
%=\Delta_{\mathrm{ind}}(\mathcal F,\xi)+\Delta h^1(\mathcal F,\xi),$ where $h^i(\mathcal F):=\dim H^i(\mathcal F)$, $\Delta h^i(\mathcal F,\xi):=h^i(\mathcal F)-h^i(\xi)$, and 
%$\Delta_{\mathrm{ind}}(\mathcal F,\xi):=\operatorname{ind}(D^+_{\mathcal F})-\operatorname{ind}(D^+_{\xi})$. The raw index jump alone does not control $H^0$: on cyclic graphs the Euler term $(r-n)\chi(G)$ may be negative even when $r>n$, making the degree-one correction $\Delta h^1$ essential. Anti-over-smoothing should therefore be measured by the relative harmonic quotient rather than by the raw index jump alone. (see Appendix ~\ref{cycle})} \begin{CJK}{UTF8}{gbsn}{\py{这个例子是不是就是我们的inverse sheaf然后identity sheaf也被包含在这里，如果是的话我就这里加一句话提一嘴详见后面的实验}}\end{CJK}
\begin{remark}[Index jump alone is insufficient]
\label{remark2}
Proposition~\ref{prop:index_jump} implies $\Delta h^0(\mathcal F,\xi)=\Delta_{\mathrm{ind}}(\mathcal F,\xi)+\Delta h^1(\mathcal F,\xi)$, where $\Delta h^i := \dim H^i(\mathcal F)-\dim H^i(\xi)$. The raw index jump $\Delta_{\mathrm{ind}}$ alone cannot control $H^0$ because the Euler term $(r-n)\chi(G)$ is often negative on cyclic graphs. Thus, the degree-one correction $\Delta h^1$ is essential.
\end{remark}
\begin{example}[Negative control]\label{example}
The Euclidean sheaf $\xi$ with trivial transport ($\mathcal{F}_{u \to e}=\mathcal{F}_{v \to e}$) yields $H^0(\xi)=\operatorname{span}\{\mathbf 1_V\}\otimes \mathbb R^n$. Although $\dim H^0(\xi)=n$ scales with feature dimension, its harmonic modes are merely channel-wise constant signals that still collapse node distinctions. Both \textit{IdentitySheaf} and \textit{InverseSheaf} exhibit this trivial inflation (see Section~\ref{sec:Experiments}, Figure~\ref{fig:oversmoothing_trajectory}).
\end{example}
As shown, large stalks or raw index jumps often merely reflect trivial channel-wise replication rather than genuine preservation of node-level geometry. The truly anti-oversmoothing requires our corrected framework: combining the index jump with the degree-one heat trace correction $\Delta h^1$, and measuring excess capacity via the relative quotient \(H^0(\mathcal F)/\iota_*H^0(\xi)\). We formalize this geometric criterion next.

%引用格式：\cite[Prop.~4.13]{spd_sheaf}\cite[Thm.~4.15]{spd_sheaf}\cite[Rem.~4.16]{spd_sheaf}

\subsection{From Index Comparison to Genuine Harmonic Inclusion}
\label{thm:index_supertrace} 
We now state the core result of this work. The detailed heat trace 
correction derivation and holonomy-fixed lower bound are deferred to 
Appendix~\ref{app:thm5}.
\begin{theorem}[\textbf{Genuine harmonic inclusion}]\label{thm:main}
Let $G$ be a connected graph. Let $\xi$ be the Euclidean sheaf of 
rank~$n$ and $\mathcal{F}$ a rank-$r$ sheaf, both constructed within the 
framework of Section~\ref{sec:framework}. Suppose a sheaf homomorphism 
$\iota:\xi\to\mathcal F$ induces an inclusion on cohomology, i.e. $\iota_\ast:H^0(\xi)\hookrightarrow H^0(\mathcal F)$.

Let $\Delta_{\mathrm{ind}}(\mathcal F,\xi):=\operatorname{ind}
(D_{\mathcal F}^{+})-\operatorname{ind}(D_{\xi}^{+})$ be the index jump 
and $\Delta_{\mathrm{Str}}^{(1)}(\mathcal F,\xi):=\lim_{t\to\infty}
\bigl[\operatorname{Tr}(e^{-tL^1_{\mathcal F}})
-\operatorname{Tr}(e^{-tL^1_{\xi}})\bigr]$ the degree-one heat trace 
correction. Assume the following conditions on $\mathcal F$:
\begin{itemize}[leftmargin=2em]
    \item \textbf{(Dimensional control).}
    The restriction maps determine a holonomy representation 
    $\rho:\pi_1(G,v_0)\to\mathrm{Aut}(\mathcal F(v_0))$, where 
    $\pi_1(G,v_0)$ is the fundamental group of $G$ at basepoint $v_0$ and $\mathrm{Aut}(\mathcal{F}(v_0))$ is the group of linear automorphisms 
of the stalk $\mathcal{F}(v_0)$.
    Its fixed subspace is
    \[
    W:=\bigcap_{\gamma\in\pi_1(G,v_0)}
    \ker(\rho(\gamma)-I)\subseteq\mathcal F(v_0),
    \qquad\dim W=k.
    \]
    Moreover, the holonomy is non-trivial on the full stalk, i.e.\ 
    $k<\operatorname{rank}(\mathcal F)$.
\item \textbf{(Capacity control).}
   At least one of the following certificates holds:
    \[
    \text{(a)}\ \Delta_{\mathrm{ind}}(\mathcal F,\xi)
    +\Delta_{\mathrm{Str}}^{(1)}(\mathcal F,\xi)\ge 1;
    \qquad
    \text{(b)}\ \dim W>\dim H^0(\xi).
    \]
\end{itemize}
Then $\dim H^0(\mathcal F)>\dim H^0(\xi)$ and the harmonic 
quotient is nontrivial: $H^0(\mathcal F)/\iota_\ast H^0(\xi)\neq 0$.

That is, $\mathcal F$ exhibits strictly greater non-dissipative harmonic capacity than the Euclidean baseline. This excess originates from non-trivial holonomy rather than trivial higher-rank channel enlargement, ensuring $\mathcal F$ provably resists Euclidean-scale over-smoothing.
\end{theorem}
\begin{remark} For small $\beta_1(G)$, after adding the "local system conditions" in Appendix~\ref{app:anti_oversmoothing_details}, certificate (a) admits a simplified explicit form: $\Delta_{\mathrm{ind}}(\mathcal F, \xi) \ge (n-k)\beta_1(G) + 1$.
Given its limited attainability, future iterations should refine $\beta_1(G)$ by synthesizing global topological invariants intrinsic to the graph geometry.

Computability establishes certificate practicability. Specifically, $\Delta_{\mathrm{ind}}$ follows by definition, $\Delta_{\mathrm{Str}}$ reduces to the discrete Laplacian trace, and (b) just requires cocycle traversal across all vertices in a simple structured graph.
\end{remark}

\subsection{Case Study: the \textit{SPDSheaf}}
\label{sec:casespd}
We now illustrate Theorem~\ref{thm:main} through a concrete comparison
between an SPD sheaf and the Euclidean baseline (details in Appendix~\ref{app:spd_verification}).

Recall that we require a sheaf $\mathcal{F}$ equipped with
stalks, restriction maps, and an admissible pairing on
cochains that defines an adjoint
$\delta_{\mathcal{F}}^{\ast}$ and the Laplacian
$\mathcal{L}_{0,\mathcal{F}}:=
\delta_{\mathcal{F}}^{\ast}\delta_{\mathcal{F}}$, with the
Hodge decomposition 
$H^{0}(\mathcal{F})\cong\ker \mathcal{L}_{0,\mathcal{F}}$. Moreover, if there exists a sheaf homomorphism
$\iota:\xi\to\mathcal{F}$, the holonomy induced by the
restriction maps is non-trivial and the index jump admits an
explicit lower bound, Theorem~\ref{thm:main} then asserts that the strict harmonic inclusion
$H^{0}(\mathcal{F})/\iota_{\ast}H^{0}(\xi)\neq 0$ holds.

Consider the SPD sheaf $\mathcal{F}_{\mathrm{SPD}}$~\cite{our}, whose stalks are $\mathrm{SPD}_n$ with congruence-type restriction maps. These determine a coboundary operator, admissible pairing, adjoint, and Hodge Laplacian satisfying the Hodge decomposition~\cite{our}. It remains to verify the three conditions of Theorem~\ref{thm:main}:

(1) \textbf{Global section preserving embedding.}
  The Euclidean-to-SPD embedding provides a sheaf morphism $\iota:\xi\to\mathcal{F}_{\mathrm{SPD}}$ that induces an injection on
  global sections:
  $\iota_{\ast}:H^{0}(\xi)\hookrightarrow
  H^{0}(\mathcal{F}_{\mathrm{SPD}})$.  Euclidean global sections are thus preserved in the SPD harmonic space
  through a concrete embedding, making the subsequent comparison
  structural rather than merely dimensional.
  
(2) \textbf{Non-trivial holonomy.}
The restriction maps induce congruence-type transports, and it generates non-trivial holonomy
  representation, ensuring the holonomy-fixed subspace $W\subsetneq\mathcal{F}_{\mathrm{SPD}}(v_{0})$. 
  
(3) \textbf{Index comparison.}
Since $\mathcal F_{\mathrm{SPD}}$ has a finite-dimensional Laplacian, 
the heat trace \(\operatorname{Tr}(e^{-tL^1_{\mathcal F_{\mathrm{SPD}}}})\) and its correction $\Delta_{\mathrm{Str}}^{(1)}$ are computable and we show in Appendix~\ref{app:spd_verification} that 
$\Delta_{\mathrm{ind}}(\mathcal{F},\xi) 
+ \Delta_{\mathrm{Str}}^{(1)}(\mathcal{F},\xi) > 0$, 
satisfying the capacity control condition of Theorem~\ref{thm:main}. The strict enlargement is then measured by 
$H^0(\mathcal F_{\mathrm{SPD}})/\iota_\ast H^0(\xi)$.

All three conditions are satisfied, so
Theorem~\ref{thm:main} yields $H^{0}(\mathcal{F}_{\mathrm{SPD}})\,/\,\iota_{\ast}H^{0}(\xi)\;\neq\;0$. This is exactly the harmonic enlargement established in
Theorem~\ref{thm:main}: the structural excess is measured as a quotient
beyond the embedded image, confirming that Theorem~\ref{thm:main} detects genuine geometric anti-oversmoothing capacity,
stronger than simply observing that a larger stalk gives a larger kernel (see Section~\ref{sec:untrained} for empirical confirmation).

\begin{remark}[Role of the SPD case]
The SPD sheaf clarifies the boundary between raw index comparison and 
harmonic-capacity comparison. Despite non-Euclidean stalks, the 
logarithmic and congruence-based structure yields a finite-dimensional 
Hodge complex after linearization, making it a concrete non-Euclidean 
application of Theorem~\ref{thm:main}.
\end{remark}

\section{Nonlinear Generalization}
\label{sec:nonlinear}
\label{nonlinearindexjump}
\subsection{ Index Jump beyond Linearity}
Let the preceding analysis relies on the linearity of the cochain complex.
When the stalks are nonlinear manifolds $M_v$, the global tools break down:
$\mathcal C^0=\prod_v M_v$ is a product manifold and $\delta$ is no longer linear.
The global section condition $\delta f=\mathbf 0$ defines a nonlinear harmonic set $\mathcal H:=\{f\in\mathcal C^0:\delta f=\mathbf 0\}$
rather than a linear kernel $\ker\delta$.
Correspondingly, the diffusion flow $\partial_t f=-L_0 f$ generalises to a
nonlinear gradient flow on $\mathcal C^0$ that is no longer generated by a
trace-class operator. As a result, the global spectral tools underpinning the linear index theorem no longer apply (details in  Appendix~\ref{app:gyrosheaf}).

The correct replacement is local: rather than requiring the entire
cochain complex to be linear, one linearises the problem at each
harmonic state $f\in\mathcal H$ by passing to the tangent space
$T_f\mathcal C^0$, a vector space.
In this sense, the nonlinear framework is a \emph{locally linearised}
version of the index-jump theory developed in the preceding sections. Concretely, around a harmonic state $f\in\mathcal H$, the nonlinear compatibility map
$\delta$ can be linearised.
Its differential $d\delta_f:T_f\mathcal C^0\longrightarrow T_{\mathbf 0}\mathcal C^1$ serves as the tangent-space analogue of the linear coboundary operator, giving rise to the \emph{tangent complex} at $f$.
The global linear cohomology groups are thus replaced by their
local tangent analogues:
\[
H^0(\mathcal F)\;\leadsto\;
  H_f^0:=\ker d\delta_f = T_f\mathcal H,
\qquad
H^1(\mathcal F)\;\leadsto\;
  H_f^1:=\operatorname{coker}d\delta_f.
\]
Correspondingly, The nonlinear analogue of the index becomes:
\[
\operatorname{ind}(D_{\mathcal F}^+)\;\leadsto\;
  \operatorname{ind}_f(d\delta)
  :=\dim\ker d\delta_f-\dim\operatorname{coker}d\delta_f.
\]

Thus the nonlinear global-section capacity is no longer measured by a global kernel dimension, but by the local dimension of the harmonic submanifold $\dim T_f\mathcal{H}=\dim\ker d\delta_f$. Accordingly, harmonic-capacity enlargement becomes a local increase of this tangent dimension rather than a global dimension jump of vector spaces; in this sense, the nonlinear framework is a locally linearised version of the index-jump theory developed in the preceding sections.

Therefore, the index--heat trace criterion applies locally after tangent linearization. At a harmonic state $f\in\mathcal{H}$, the nonlinear coboundary $\delta$ has differential $d\delta_f:T_fC^0\to T_{\delta f}C^1$ with local analytic index $\operatorname{ind}_f(d\delta):=\dim\ker d\delta_f-\dim\operatorname{coker}d\delta_f$. The local analogue of Theorem~\ref{thm:main} is $\Delta \dim H_f^0=\Delta_{\mathrm{ind},f}+\Delta_{\mathrm{Str},f}^{(1)}$, where $\Delta_{\mathrm{Str},f}^{(1)}$ involves calculation of trace for analytic operator which we have discussed before. Thus, for a nonlinear sheaf, it suffices to construct $d\delta_f$, equip the tangent cochains with admissible pairings, define the adjoint and tangent Hodge Laplacians, and verify the tangent Hodge decomposition. We realize this program through the SPD-based \textit{GyroSheaf} framework below.

\subsection{Case Study: the \textit{GyroSheaf}}\label{gyrosheaf}
\paragraph{From SPD matrices to the Poincar\'e ball.}
We first map each SPD feature $P_v\in\mathrm{SPD}_n$ to the Poincar\'e ball $\mathcal{B}_n$ via the Cayley transformation:
\[
  X_v \;:=\; \Psi(P_v) \;=\; (P_v - I)(P_v + I)^{-1} 
  \;\in\; \mathcal{B}_n 
  \;:=\; \{X = X^\top : \|X\| < 1\}.
\]
This yields a bounded gyrovector domain in which edge-wise inconsistencies can be measured through gyro-algebraic operations while preserving SPD geometry.

\begin{definition}[Restriction maps]
\label{gyroR}
A \emph{gyro-valued sheaf} $(G, \mathcal{F})$ on a graph $G = (V, E)$ assigns the Cayley ball
$\mathcal F(v)=\mathcal B_n$ to each vertex $v$ and the symmetric-matrix space
$\mathcal F(e)=\operatorname{Sym}_n$ to each edge $e$. For each incident pair $v \in e$, the restriction map $\mathcal{F}_{v \to e}: \mathcal F(v) \to \mathcal F(e)$ is defined as:
\[
X_{ve} := \mathcal{F}_{v\to e}(X_v)=M_{ve}X_vM_{ve}^{\top},
    \qquad
    M_{ve}\in O(n).
\]
\end{definition}
The differential $d\mathcal{F}_{v\to e}:T_{X_v}\mathcal F(v)\to T_{X_{ve}}\mathcal F(e)$
where the tangent spaces both canonically identified with $\operatorname{Sym}_n$ and the Frobenius adjoint act as:
\[
  d\mathcal{F}_{v\to e}(G)=M_{ve}\,G\,M_{ve}^\top \quad G \in T_{X_v}\mathcal F(v) ;\qquad
  (d\mathcal{F}_{v\to e})^*(G')=M_{ve}^\top\,G'\,M_{ve}\quad G' \in T_{X_{ve}}\mathcal F(e).
\]

\textbf{Gyro-Coboundary operator.} Since matrix multiplication is non-commutative, we adopt the symmetrised gyro-subtraction to measure edge-wise inconsistency:
\[
    A \ominus_{sym} B := (I - AB)^{-1/2} (A - B) (I - BA)^{-1/2}\in\operatorname{Sym}_n, \quad A, B \in \mathcal{B}_n .
\]

\begin{definition}[Gyro-compatible Coboundary Operator]
The coboundary operator $\delta: C^0(G;\mathcal{F})\longrightarrow C^1(G;\mathcal{F})$, with $C^0(G;\mathcal{F})=\prod_{v\in V}\mathcal{B}_n$ and $C^1(G;\mathcal{F})=\bigoplus_{e\in E}\mathrm{Sym}(n)$, is defined on each oriented edge $e=(u,v)$ as:
\[
    \delta(X)_e = \mathcal{F}_{v \to e}(X_v)\ominus \mathcal{F}_{u \to e}(X_u)  = X_{ve} \ominus_{sym} X_{ue} = S_{ve}^{1/2} (X_{ve} - X_{ue}) S_{ue}^{1/2} 
\in\operatorname{Sym}_n=\mathcal F(e),\]
where we denote $S_{ve} := (I - X_{ve}X_{ue})^{-1}$ and $S_{ue} := (I - X_{ue}X_{ve})^{-1} = S_{ve}^\top$ for notational brevity.
\end{definition}

\paragraph{Differential and adjoint.}
\label{differentil}
Since $\delta$ is nonlinear, the adjoint appearing in the gyro-sheaf Laplacian is the Frobenius adjoint of the \emph{linearised} coboundary at the current state $X$. Its differential and Frobenius adjoint have types
$D\delta(X): \bigoplus_{v\in V} T_{X_v}\mathcal{B}_n \longrightarrow C^1(G;\mathcal{F})$ and $D\delta(X)^*: C^1(G;\mathcal{F}) \longrightarrow \bigoplus_{v\in V} T_{X_v}\mathcal{B}_n.$ Their closed-form expressions involve the Fréchet derivative of the principal
matrix inverse square root, and are deferred to Appendix~\ref{calculation}. In implementation, we compute the
gyro-sheaf Laplacian directly as the gradient of the Dirichlet energy.

\begin{definition}[Dirichlet energy]Let $(\delta X)_e=\Delta_e$. The energy on edge $e$ and the total energy are defined as $\mathcal{E}_e(X):=\tfrac{1}{2}\|\Delta_e\|_F^2$; $
  \mathcal{E}(X):=\sum_{e\in E}\mathcal{E}_e(X)$.
\end{definition}

\begin{definition}[Gyro-sheaf Laplacian]\label{def:gyro_laplace}
    For a directed edge $e=(u,v)$ with $v$ as the head, the gyro-sheaf Laplacian at vertex $v$ is the gradient of the total inconsistency energy $(\mathcal L X)_v
:=
\nabla_{X_v}\mathcal E(X)
\in
T_{X_v}\mathcal B_n
=
\operatorname{Sym}(n)$. Recall that in the linear sheaf setting the coboundary
$\delta$ is a linear and the sheaf Laplacian equals
$\mathcal{L}=\delta^{\ast}\delta$.
In our setting, $\delta$ is nonlinear, we replace $\delta$ with its tangent map
$D\delta(X)$, and the Laplacian generalises accordingly to
$\mathcal{L}X = \bigl(D\delta(X)\bigr)^*D\delta(X)$.
More explicitly,
\[\begin{aligned}
  (\mathcal{L}X)_v =\sum_{e=(u,v)} \left( d\mathcal{F}_{v \to e} \right)^* \left( d\delta_{ve} \right)^* \delta(X)_e
  +\sum_{e=(v,w)} \left( d\mathcal{F}_{v \to e} \right)^* \left( d\delta_{ve} \right)^*\delta(X)_e
\end{aligned}
\]
\end{definition}

\begin{proposition}[Hodge-type Decomposition for the Nonlinear GyroSheaf]\label{prop:Gyro-Valued Hodge}
Let $T_X\mathcal{C}^0=\bigoplus_{v\in V}T_{X_v}\mathcal{B}_v$ and
$T_{\delta(X)}\mathcal{C}^1=\bigoplus_{e\in E}T_{\delta(X)_e}\mathcal{B}_e$
be the tangent-space direct sums, and let
$D\delta_X:T_X\mathcal{C}^0\to T_{\delta(X)}\mathcal{C}^1$ be the
linearised coboundary. Under the Frobenius pairing
$\langle A,B\rangle_F=\operatorname{tr}(A^\top B)$, $D\delta_X$ is a
linear operator between Hilbert spaces. Then:
\[
  T_{\delta(X)}\mathcal{C}^1
  =\operatorname{im}(D\delta_X)\oplus\ker(D\delta_X^*),\qquad
  T_X\mathcal{C}^0
  =\operatorname{im}(D\delta_X^*)\oplus\ker(D\delta_X).
\](detalis in Appendix~\ref{app:D3})
\end{proposition}

\begin{remark}
The GyroSheaf verifies the local Hodge-theoretic structure required by the nonlinear index--heat trace framework: it constructs the tangent complex, defines the tangent Hodge Laplacian, and proves the corresponding tangent Hodge decomposition, yielding a well-defined local harmonic space. Moreover, since its restriction maps act by $O(n)$-congruence on $\mathrm{Sym}_n$ (Definition~\ref{gyroR}), the tangent complex at any $f\in\mathcal{H}$ inherits the holonomy and index data of the linearised SPD complex (Section~\ref{sec:casespd}), so the conditions of Theorem~\ref{thm:main} are satisfied by identical verification. Experiments in Section~\ref{sec:Experiments} and Appendix~\ref{app:additional_experiment} confirm the predicted anti-oversmoothing behavior.
\end{remark}

\section{Experiments}
\label{sec:Experiments}
% Our experiments empirically validate the central prediction of the index-theoretic framework: the anti-oversmoothing capacity of a sheaf diffusion model is governed by the dimension of its harmonic section space, and Theorem~\ref{thm:main} induces a strict ordering of harmonic capacities across model families. We evaluate both untrained propagation (Sec.~\ref{sec:untrained}), which isolates operator-level structure, and trained classifiers (Sec.~\ref{sec:trained}), which verify that these structural distinctions persist under end-to-end optimization.
Our experiments empirically validate the predictions of the index-theoretic framework. The design is structured as a progressive verification of Theorem~\ref{thm:main}: we move from GCNs, which correspond to a trivial sheaf (Baseline), through sheaf models that violate the theorem's conditions (Tier~1), to sheaf models that satisfy them with increasing index jump (Tiers~2--3).This progression directly tests whether the theorem govern anti-oversmoothing capacity as predicted. Datasets, models, training protocol, and oversmoothing measures are detailed in Appendix~\ref{app:experimental_setup}. Code is available \href{https://anonymous.4open.science/r/GyroSheaf-6514}{here}.

\subsection{Models}
\textit{\textbf{Baseline}: rank-1 trivial sheaf.} GCN~\cite{kipf2017gcn}, GAT~\cite{velickovic2018gat}, GraphSAGE~\cite{hamilton2017graphsage}. Standard GNN propagation by the graph Laplacian~$L_G$, corresponding to the rank-1 trivial sheaf (scalar stalk $\mathcal{F}(v) = \mathbb{R}$, identity restriction maps). No index jump, no sheaf structure. Oversmoothing is the default. 

\textit{\textbf{Tier~1}: sheaf structure with index jump, but conditions violated.} Two negative controls that possess sheaf structure with stalk dimension~$d$ and a nonzero index jump $\Delta\mathrm{ind} = (d-1)\chi(G)$ over the baseline, as exemplified after Prop~\ref{prop:index_jump}, are predicted by Theorem~\ref{thm:main} to fail at harmonic enlargement:
% \begin{CJK}{UTF8}{gbsn}{
% \py{Ai写的下面这个，之后改}
% }\end{CJK}
\begin{itemize}[leftmargin=2em]
\item \textbf{IdentitySheaf}~\cite{recent_critique_2026}: Consider the sheaf
  in which every restriction map is the identity~$I_d$.  The parallel
  transport along any edge is therefore also the identity, so the holonomy
  representation is trivial: $\rho(\gamma)=I_d$ for every
  $\gamma\in\pi_1(G,v_0)$.  Consequently,
$W=\bigcap_\gamma\ker(\rho(\gamma)-I)=\mathcal{F}(v_0)$, i.e.\ the
  holonomy-fixed subspace equals the entire stalk.  Moreover, the trivial
  holonomy yields a global trivialisation of sheaf structure, so the harmonic
  space reduces to a rank-$d$ tensor copy of the scalar Euclidean baseline
  (channel-wise constant signals).  The hypotheses of
  Theorem~\ref{thm:main} are therefore \emph{not satisfied}: Condition~1
  (Dimensional control) requires non-trivial holonomy.
\item \textbf{InverseSheaf}: For each edge~$e$, the restriction maps from
  the two endpoints to~$e$ are the same learned invertible matrix~$M_e$.  Although the restriction
  maps appear non-trivial, the
  parallel transport from one endpoint to the other is
  $M_e^{-1}M_e = I_d$, which renders the holonomy around every cycle
  trivial.  There situation thus reduces to the IdentitySheaf case above.
\end{itemize}

\textit{\textbf{Tier~2}: conditions satisfied, Euclidean stalks.} DiagSheaf, BundleSheaf, and GeneralSheaf~\cite{bodnar2022nsd}, with diagonal, orthogonal, and general restriction maps respectively. Same stalk dimension~$d$ and index jump as Tier~1, but their learned restriction maps can develop holonomy in the middle ground required by Condition~1, non-trivial yet preserving enough fixed directions, while also satisfying Condition~2.
 
\textit{\textbf{Tier~3}: conditions satisfied, higher-dimensional non-Euclidean stalks.} SPDSheaf~\cite{our} and our \textbf{GyroSheaf}, whose manifold-valued stalks have effective dimension $r = d(d{+}1)/2 > d$, which enlarges the ambient space in which holonomy-fixed directions may occur. On cyclic graphs, this yields a strict enlarge of relevant quantity \(\Delta h^0=\Delta_{\mathrm{ind}}+\Delta_{\mathrm{Str}}^{(1)}\). Both conditions are satisfied.
% Sample figure used in the Figures subsection.
% Replace \fbox{...} with \includegraphics{your_file.pdf} once you have one.
\begin{figure}[t]
  \centering
  \includegraphics[width=0.85\linewidth]{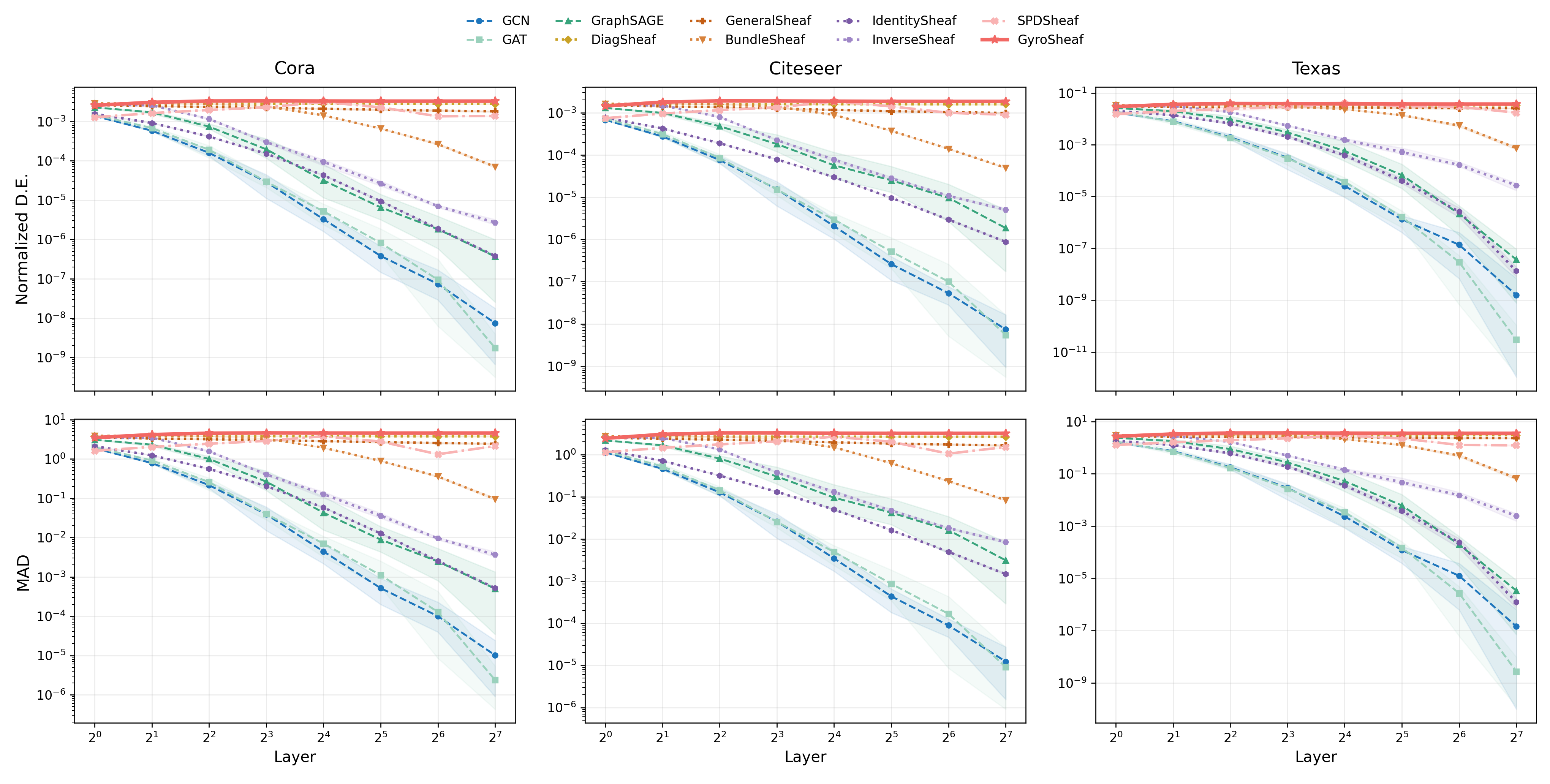}
    \caption{Normalized Dirichlet energy (top) and MAD (bottom) on Cora, Citeseer, and Texas in the untrained regime. Curves and shaded bands show per-seed mean and min–max range over five seeds.}
    % \textcolor{blue}
% {
% % \begin{CJK}{UTF8}{gbsn}{
% % inverse和identity表现不好：用index jump解释？，或者是和理论match的地方 以及bundle和其他sheaf都是欧氏空间，为什么就bundle表现的不好，怎么解释这里，怎么用holonomy那里去解释}
% % \end{CJK}
% % }
% }
  \label{fig:oversmoothing_trajectory}
\end{figure}

\subsection{Untrained Regime: Structural Diagnostic}
\label{sec:untrained}
Figure~\ref{fig:oversmoothing_trajectory} reports normalized Dirichlet energy and MAD on Cora, Citeseer, and Texas. The trajectories confirm the four-level ordering predicted by Theorem~\ref{thm:main}. Baseline and both Tier~1 negative controls collapse by six to eight orders of magnitude within 128~layers---despite Tier 1's enlarged stalks and nontrivial sheaf-like structure, they decay to the \emph{same level} as the rank-1 baselines, confirming that index jump without the holonomy conditions is insufficient. The Tier~2 learnable sheaves substantially delay the decay (with BundleSheaf slightly faster, consistent with the orthogonal constraint limiting holonomy flexibility), while the Tier~3 geometric sheaves maintain both metrics essentially flat across all depths and datasets, reflecting the combined effect of satisfied conditions and a larger index jump.

\begin{figure}
    \centering
    \includegraphics[width=0.85\linewidth]{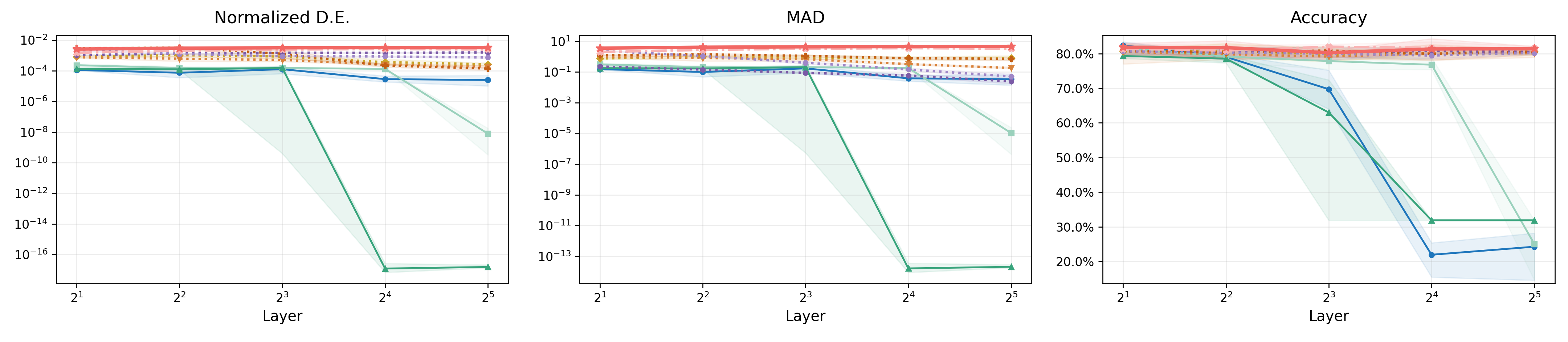}
    \caption{Normalized Dirichlet energy 
(left), MAD (middle), and test accuracy (right) on Cora in trained regime. Legend and shading as in Figure~\ref{fig:oversmoothing_trajectory}.}
    \label{fig:oversmoothing_trajectory2}
\end{figure}

\subsection{Trained Regime: End-to-End Confirmation}
\label{sec:trained}
 
Figure~\ref{fig:oversmoothing_trajectory2} reports normalized Dirichlet energy, MAD, and test accuracy on Cora across depths. Baseline accuracy drops sharply beyond $L{=}8$, while all sheaf models, including Tier~1 controls, maintain ${\sim}80\%$ accuracy, confirming that the enlarged harmonic capacity of Tiers~2--3 does not come at the cost of finite-depth performance. The structural distinctions between Tiers~1 and~2, clearly separated in the untrained regime, are masked by finite-depth optimization in the trained regime, which is expected: the operator-level criterion of Theorem~\ref{thm:main} characterizes intrinsic capacity in the infinite-depth limit, complementary to finite-depth trained behavior. We also note that GCN's Dirichlet energy does not decay substantially despite its accuracy degradation, consistent with the bias-induced energy plateau in \cite{rusch2023survey}'s Figure 3, further illustrating that trained-regime energy trajectories do not directly reflect operator-level oversmoothing behavior. Together, these observations confirm that the untrained diagnostic is the appropriate probe for the operator-level harmonic capacity of Theorem~\ref{thm:main}.
\section{Conclusion}
We presented the first relative criterion for anti-oversmoothing capacity in Sheaf Neural Networks. Theorem~\ref{thm:main} establishes that a positive index jump, combined with natural conditions on holonomy and cohomological structure, guarantees genuine harmonic enlargement beyond the Euclidean baseline. We extended this framework to nonlinear stalks via local linearisation and instantiated it with GyroSheaf. Experiments across ten models confirm the criterion.

% ---- Acknowledgments & References ----
% \begin{ack}
%   \input{sections/acknowledgments}
% \end{ack}

%\section*{References}
\medskip
{\small
\bibliographystyle{plainnat} 
% 或 unsrtnat / ieeetr / abbrvnat 等
\bibliography{bib/references}  % 对应 bib/references.bib，不要写 .bib 后缀
}

% ---- Appendix ----
\newpage
\appendix
\section{Limitations}
\label{app:limitation}
While our work introduces the first relative, index-theoretic criterion for anti-oversmoothing capacity in Sheaf Neural Networks and validates it across ten models on both homophilic and heterophilic benchmarks, it presents some limitations and areas for future exploration. Our theoretical framework and experiments focus on standard graph structures; extending the index-theoretic criterion to more general combinatorial domains such as hypergraphs and cell complexes, where the coboundary operators and Hodge decomposition have known analogues, is a natural direction, though verifying the holonomy and cohomological conditions of Theorem~\ref{thm:main} in these settings would require careful adaptations. Additionally, as our analysis demonstrates, the operator-level criterion of Theorem~\ref{thm:main} precisely characterises intrinsic harmonic capacity in the infinite-depth limit, which serves a complementary role to finite-depth trained behaviour—the structural distinctions clearly separated in the untrained regime can be partially masked by finite-depth optimisation, as expected from an operator-theoretic guarantee. Investigating how the index-theoretic criterion interacts with training dynamics to further inform architectural choices in practical finite-depth settings is a promising direction for future work.

\section{Details in Section 3}\label{app: B}
\subsection{Details of ~\ref{sec:1} (Sheaves on Graphs)}
To bridge the gap between the continuous fiber bundle $\pi: \mathcal{E} \to M$ and the graph-based construction, we formalize the following process.

\textbf{Description of stalks and maps.}
The graph $G=(V,E)$ is viewed as a 1-dimensional cell complex sampled from the manifold $M$. 
The node stalks $\mathcal{F}_v$ are the fibers $\pi^{-1}(v)$ of the continuous bundle.
 For an edge $e = (u,v)$, the edge stalk $\mathcal{F}_e$ is the fiber over the midpoint of $e$ or a shared representative space.
The restriction maps $\mathcal{F}_{u \to e}$ and $\mathcal{F}_{v \to e}$ are the discrete realizations of the connection form $\omega$ defined in Section 2.1. Specifically, $\mathcal{F}_{v \to e}$ represents the parallel transport from $v$ onto the edge $e$.

\textbf{Cochains and operator grounding.}
The space of $0$-cochains $\mathcal{C}^0(G, \mathcal{F})$ is the discrete analog of the space of smooth sections $\Gamma(M, \mathcal{E})$. A $0$-cochain $f \in \mathcal{C}^0$ assigns a local feature $f(v) \in \mathcal{F}_v$ to each vertex. The $1$-cochains $\mathcal{C}^1(G, \mathcal{F})$ capture the "difference" of these features across edges, which corresponds to $\Omega^1(M, s^*\mathcal{V})$. The discrete coboundary operator $\delta: \mathcal{C}^0 \to \mathcal{C}^1$ measures the failure of a $0$-cochain to be a global parallel section, directly simulating the continuous operator $d_0 s = s^* \omega$.

\subsection{Details of ~\ref{sec:2} (Coboundary Operator)}

This section provides the rigorous formulation of the discrete coboundary operator introduced in Definition \ref{def: def2}, emphasizing its role as a measure of local incompatibility on non-linear stalks.

\textbf{Geometric interpretation of $\delta$.}
The discrete coboundary operator $\delta$ serves as the discrete analogue of the connection-induced differential $d_0 s = s^* \omega$. For a $0$-cochain $f \in \mathcal{C}^0(G, \mathcal{F})$ (which assigns a local feature $f(v) \in \mathcal{F}_v$ to each node), the operator evaluates the failure of $f$ to form a globally parallel section. Along any edge $e=(u,v)$, the maps $\mathcal{F}_{u \to e}$ and $\mathcal{F}_{v \to e}$ transport the node features into the common edge stalk $\mathcal{F}_e$. The quantity $(\delta f)_e$ must represent the intrinsic geometric difference between these transported features. 

As in the remark after Definition~\ref{def: def2}, the output $(\delta f)_e$ technically lies in a general type space with no natural algebra structure on it, the key point of defining $\delta$ is how to formulate the ``difference'' in terms of specific sheaf types. One typical approach is to try to identify the algebraic structure inherent in the space, and another way is utilizing local linearize technique to identify algebraic structure on tangent space.

\begin{example}
[Linear sheaf coboundary]
One can verify above discussion by its behavior on classical linear sheaves. Suppose the stalks $\mathcal{F}_v$ and $\mathcal{F}_e$ are Euclidean vector spaces, and the restriction maps are linear transformations. The difference natively becomes the standard vector difference. In this case, the operator reduces exactly to:
\[ (\delta f)_{u \to v} = \mathcal{F}_{v \to e} f(v) - \mathcal{F}_{u \to e} f(u) \]
which is precisely the classical definition of the degree-$0$ coboundary operator in Euclidian sheaf and other previous work~\cite{SNN2020, bodnar2022nsd, hansen2019toward}.
\end{example}

\subsection{Proof of Proposition~\ref{prop1} (Energy-Laplacian Correspondence)}
We prove the three properties connecting the sheaf Laplacian $\mathcal{L} = \delta^\ast \delta$ to the Dirichlet energy and diffusion dynamics.
\begin{proof}
	\begin{itemize}[leftmargin=2em]
		\item \textit{Step 1: Gradient of the discrete energy.}
		Given the admissible pairing $\langle \cdot, \cdot \rangle$ on the cochain spaces, the discrete energy is $\mathcal{E}(x) = \frac{1}{2} \langle \delta x, \delta x \rangle$. We consider the first variation of this energy along an arbitrary test $0$-cochain $v$:
		\[ \delta \mathcal{E}(x)[v] = \left. \frac{d}{d\epsilon} \left( \frac{1}{2} \langle \delta(x + \epsilon v), \delta(x + \epsilon v) \rangle \right) \right|_{\epsilon=0} = \langle \delta x, \delta v \rangle \]
		By the defining property of the adjoint operator $\delta^\ast$, we can move the coboundary operator to the left side of the pairing:
		\[ \langle \delta x, \delta v \rangle = \langle \delta^\ast \delta x, v \rangle = \langle \mathcal{L}x, v \rangle \]
		Since $\delta \mathcal{E}(x)[v] = \langle \nabla \mathcal{E}(x), v \rangle$ holds for all variations $v$, we conclude that the gradient of the energy functional is exactly $\nabla \mathcal{E}(x) = \mathcal{L}x$.
		\item \textit{Step 2: Measurement of local inconsistency (ii).}
		To see how $\mathcal{L}$ encodes the total inconsistency, we evaluate the inner product of the Laplacian applied to $x$ with $x$ itself:
		\[ \langle \mathcal{L}x, x \rangle = \langle \delta^\ast \delta x, x \rangle \]
		Applying the adjoint relation again, we obtain:
		\[ \langle \delta^\ast \delta x, x \rangle = \langle \delta x, \delta x \rangle = \|\delta x\|^2 \]
		This confirms that the quadratic form of the Laplacian exactly recovers the accumulated edgewise geometric discrepancies.
		\item \textit{Step 3: Positivity and diffusion stability (iii).}
		From Step 2, we have $\langle \mathcal{L}x, x \rangle = \|\delta x\|^2 \ge 0$ for any $x \in \mathcal{C}^0$. Thus, $\mathcal{L}$ is inherently positive semi-definite. When the system evolves according to the diffusion equation $\partial_t x_t = -\mathcal{L}x_t$, the time evolution of the energy is given by:
		\[ \frac{d}{dt} \mathcal{E}(x_t) = \langle \nabla \mathcal{E}(x_t), \partial_t x_t \rangle = \langle \mathcal{L}x_t, -\mathcal{L}x_t \rangle = -\|\mathcal{L}x_t\|^2 \le 0 \]
		This strictly non-positive derivative guarantees that the diffusion process dissipates the inconsistency energy monotonically, driving the signal toward a stable harmonic state where $\mathcal{L}x_t = 0$.
	\end{itemize}
\end{proof}

\subsection{Proof of Theorem~\ref{thm:thm2} (0th-order Discrete Hodge Decomposition)}
We state the Hodge decomposition for 
0-cochains on discrete sheaves, where the cochain pairing and its adjoint are already defined. This proof establishes the identity between the discrete harmonic space and the space of global sections by demonstrating the equivalence of their kernels.
\begin{proof}
	\begin{itemize}[leftmargin=2em]
		 \item \textit{Step 1: The identity $\ker \mathcal{L} = \ker \delta$.} 
		We prove the equivalence through double inclusion based on the positive definiteness of the admissible pairing.
		\begin{itemize}[leftmargin=2em]
			\item $\supseteq:$ Let $f \in \ker \delta$, such that $\delta f = 0$. By the definition of the discrete Laplacian $\mathcal{L} = \delta^\ast \delta$:
			\[ \mathcal{L} f = \delta^\ast (\delta f) = \delta^\ast(0) = 0 \]
			This implies $f \in \ker \mathcal{L}$.
			\item $\subseteq:$ Conversely, let $f \in \ker \mathcal{L}$, meaning $\mathcal{L} f = 0$. We evaluate the quadratic form:
			\[ 0 = \langle \mathcal{L}f, f \rangle = \langle \delta^\ast \delta f, f \rangle \]
			Applying the definition of the adjoint operator $\delta^\ast$:
			\[ \langle \delta^\ast \delta f, f \rangle = \langle \delta f, \delta f \rangle  \]
			Since the pairing is admissible, $\langle \delta f, \delta f \rangle = 0$ implies $\delta f = 0$ exactly. Thus, $f \in \ker \delta$.
		\end{itemize}
		
			\item \textit{Step 2: Connection to Sheaf Cohomology.}
		By the definition of the discrete sheaf structure, a $0$-cochain $f$ is a global section if it is compatible with all restriction maps, which is precisely the condition $(\delta f)_e = 0$ for all $e \in E$. In the language of sheaf cohomology, the $0$-th cohomology group is defined as the kernel of the first differential:
		\[ H^0(G, \mathcal{F}) = \ker \delta \]
	\end{itemize}
	Combining Step 1 and Step 2, we obtain the isomorphism $\ker \mathcal{L} \cong H^0(G, \mathcal{F})$.
	This result ensures that the steady states of the discrete diffusion process $\partial_t x = -\mathcal{L}x$ are precisely the global sections of the sheaf.
\end{proof}

\subsection{Technical Details for Proposition~\ref{prop:anti_oversmoothing_principle} (Harmonic Characterisation of OverSmoothing)}
\label{app:anti_oversmoothing_details}

 The principle stated in Proposition~\ref{prop:anti_oversmoothing_principle} provides an operator-theoretic analysis of representation collapse.

The discrete diffusion process is governed by the linear system $\partial_t f(t) = -\mathcal{L} f(t)$. Given that $\mathcal{L} = \delta^* \delta$ is a symmetric positive semi-definite operator, let $\{(\lambda_i, \phi_i)\}_{i=1}^n$ be its eigenpairs, where $0 = \lambda_1 = \cdots = \lambda_k < \lambda_{k+1} \le \cdots \le \lambda_n$. The general solution for the feature evolution is:
\[ 
f(t) = \sum_{i=1}^k \langle f^{(0)}, \phi_i \rangle \phi_i + \sum_{j=k+1}^n e^{-\lambda_j t} \langle f^{(0)}, \phi_j \rangle \phi_j 
\]
As $t \to \infty$, the terms associated with strictly positive eigenvalues vanish. The long-time limit is exactly the projection of the initial state onto the kernel of $L$, which coincides with previous discussion of Hodge decomposition:
\[ 
f^{(\infty)} = \lim_{t \to \infty} f(t) = \sum_{i=1}^k \langle f^{(0)}, \phi_i \rangle \phi_i = P_{\mathcal{H}^0} f^{(0)} 
\]
When the input signal is centered or nearly centered ($\mu \approx 0$), the ratio simplifies to $\mathbb{E}[\mathcal{R}(\mathcal{F})] \approx 1/N$. This provides a rigorous quantitative characterization of operator-induced oversmoothing: as the graph size $N$ increases, the portion of discriminative information preserved by standard diffusion vanishes at a rate of $\mathcal{O}(1/N)$. Our framework aims to mitigate this by enriching $\mathcal{H}^0$ such that $\dim(\mathcal{H}^0) \gg 1$, effectively lifting the preservation ratio away f rom this zero-limit.

\begin{figure}[H]
	\centering
	\includegraphics[width=0.5\textwidth]{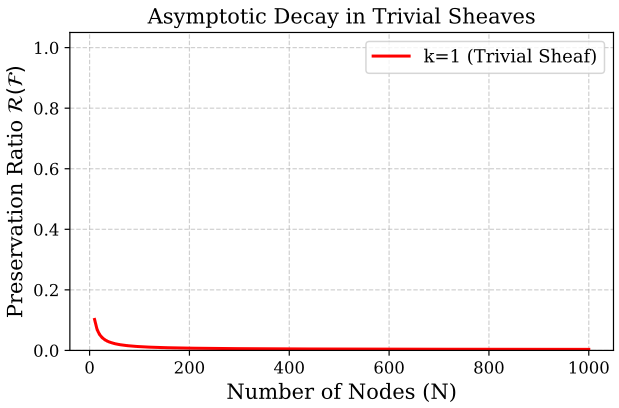}
	\caption{ Asymptotic decay of $\mathcal{R}(\mathcal{F})$ in trivial sheaves ($dim(\mathcal{H}^0)=1)$. }
\end{figure}
\begin{figure}[H]
	\centering
	\includegraphics[width=0.5\textwidth]{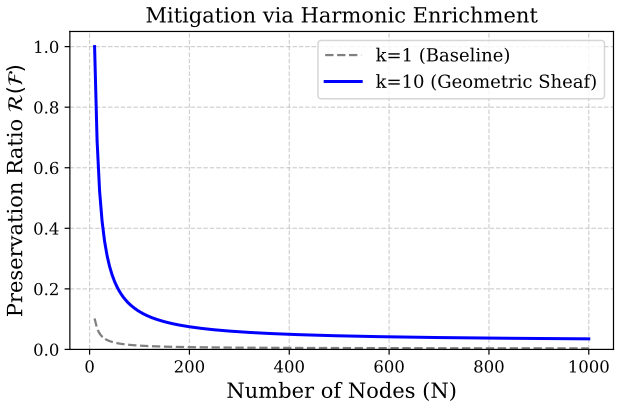}
	\caption{ Mitigation of representation collapse through intermediate harmonic enrichment. }
\end{figure}
\begin{figure}[H]
	\centering
	\includegraphics[width=0.5\textwidth]{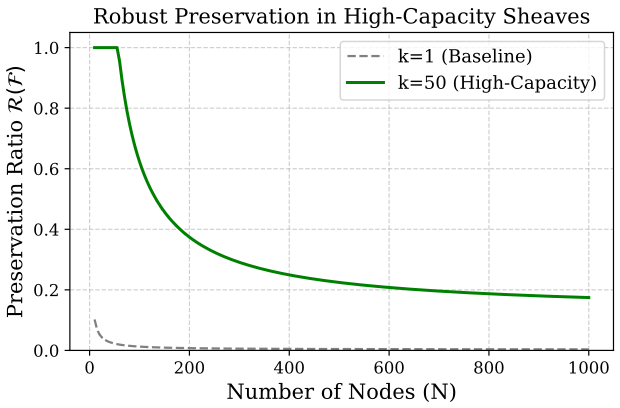}
	\caption{ Robust information preservation in high-capacity geometric sheaf frameworks. }
\end{figure}

\textbf{Two senses of harmonic enrichment.}
The capacity of a sheaf framework to resist oversmoothing is determined by the size of its harmonic space. We formalize this enrichment in two mathematically distinct senses:
\begin{enumerate}[leftmargin=2em]
	\item \textit{Weak Enrichment (Dimensional Capacity):} A sheaf $\mathcal{F}_2$ is weakly richer than $\mathcal{F}_1$ if $\dim(\mathcal{H}^0_2) > \dim(\mathcal{H}^0_1)$. 
	\item \textit{Strong Enrichment (Subspace Inclusion):} A sheaf $\mathcal{F}_2$ is strongly richer if $\mathcal{H}^0_1 \subsetneq \mathcal{H}^0_2$. Under this structural condition, we obtain a deterministic, point-wise guarantee. By the decomposition of subspaces ($\mathcal{H}^0_2 = \mathcal{H}^0_1 \oplus W$), we have $\|P_{\mathcal{H}^0_2} f^{(0)}\|^2 \ge \|P_{\mathcal{H}^0_1} f^{(0)}\|^2$ for \textit{any} arbitrary input $f^{(0)}$.
\end{enumerate}
Standard graph diffusion restricts $\mathcal{H}^0$ to the 1-dimensional constant subspace. By employing geometry-aware restriction maps, our framework seeks to achieve strict structural enrichment (strong enrichment), lifting the preservation ratio deterministically and unconditionally. This sets the stage for the rigorous index improvement theory developed in Section~\ref{subsec:index_improvement_framework}.

\section{An Index-Theoretic Analysis of oversmoothing}\label{app:C}
\subsection{Algebraic Foundations of the Discrete Index~\ref{subsec:index_improvement_framework}}\label{app:C.1}

\textbf{The discrete Hodge--Dirac framework.}
This appendix provides the rigorous algebraic proofs for the discrete sheaf index and its holonomy-induced harmonic expansion.

Recall the $0$-th order discrete Hodge decomposition $\mathcal{C}^0 = \ker \mathcal{L} \oplus \mathrm{Im} \delta^\ast$. The isomorphism $H^0(G, \mathcal{F}) \cong \ker \mathcal{L}$ identifies the harmonic space with the $0$-th cohomology group. Consequently, the discrete index explicitly constrains the harmonic dimension:
\begin{equation}
	\dim \ker \mathcal{L} = \mathrm{Ind}(\mathcal{F}) + \dim \ker \delta^\ast. \label{eq:harmonic_capacity}
\end{equation}
This formula establishes the structural link between the analytic capacity of the diffusion operator and the topological invariants of the sheaf.

The stability of $\mathrm{Ind}(\mathcal{F})$ under the diffusion flow $\partial_t f = -Lf$ is rigorously verified via the discrete heat kernel supertrace (see, e.g., \cite{BGV} for general supertrace properties). Let $\mathcal{L}_0 = \delta^\ast \delta$ and $\mathcal{L}_1 = \delta \delta^\ast$. For any $t > 0$, the supertrace is defined as:
\begin{equation}
	\mathrm{Str}(e^{-tL}) = \mathrm{Tr}(e^{-t L_0}) - \mathrm{Tr}(e^{-t L_1}).
\end{equation}
Let $\{ \lambda_i \}$ be the strictly positive eigenvalues of $\mathcal{L}_0$. The coboundary operator $\delta$ acts as an isomorphism between the positive eigenspaces $E_{\lambda_i}(\mathcal{L}_0)$ and $E_{\lambda_i}(\mathcal{L}_1)$, since the intertwining relation $\delta \mathcal{L}_0 = \mathcal{L}_1 \delta$ holds. Consequently, the spectral contributions from all strictly positive eigenvalues cancel exactly in the alternating sum:
\begin{align}
	\mathrm{Str}(e^{-t\mathcal{L}}) &= \left( \dim \ker \mathcal{L}_0 + \sum_{\lambda_i > 0} e^{-t \lambda_i} \right) - \left( \dim \ker \mathcal{L}_1 + \sum_{\lambda_i > 0} e^{-t \lambda_i} \right) \\
	&= \dim \ker \mathcal{L}_0 - \dim \ker \mathcal{L}_1.
\end{align}
Since $\ker \mathcal{L}_0 = \ker \delta$ and $\ker \mathcal{L}_1 = \ker \delta^\ast$, this yields $\mathrm{Str}(e^{-t\mathcal{L}}) = \mathrm{Ind}(\mathcal{F})$. Taking the limit $t \to \infty$ confirms that the deep-layer representational capacity is intrinsically locked to the topological index, invariant to the diffusion time.

In deep graph neural networks, oversmoothing refers to the exponential decay of node features into a trivial subspace that carries no discriminative signal. Modeling forward propagation as continuous-time heat diffusion $\partial_t f = -\mathcal{L}f$ with graph Laplacian $\mathcal{L}$, the asymptotic representation capacity is fully characterized by $\dim\ker \mathcal{L}$.

To analyze this capacity, we lift graph signals to a cellular sheaf $\mathcal{F}$ over a graph $G=(V,E)$. Denote by $\mathcal{C}^0 = \bigoplus_{v\in V}\mathcal{F}(v)$ and $\mathcal{C}^1 = \bigoplus_{e\in E}\mathcal{F}(e)$ the spaces of $0$-cochains (node features) and $1$-cochains (edge features), respectively, and let $\operatorname{rank}(\mathcal{F})=r$.

\begin{definition}[Hodge--Dirac operator]
The coboundary operator $\delta\colon\mathcal{C}^0\to\mathcal{C}^1$ measures local inconsistency across edges. Given an admissible inner product, its adjoint $\delta^\ast\colon\mathcal{C}^1\to\mathcal{C}^0$ acts as the discrete divergence, and the $0$-Laplacian is $\mathcal{L}_0 = \delta^\ast\delta$. The discrete Dirac operator on $\mathcal{C}^0\oplus\mathcal{C}^1$ is $D_{\mathcal{F}} = \delta+\delta^\ast$, with chiral component $D_{\mathcal{F}}^{+} = \delta\colon\mathcal{C}^0\to\mathcal{C}^1$.
\end{definition}

By discrete Hodge theory, $\ker L_0 = \ker\delta \cong H^0(\mathcal{F})$, which identifies the space of global sections, i.e., parallel node signals. The first cohomology $H^1(\mathcal{F})$ captures harmonic edge flows. In the graph-level case where $\mathcal{C}^2=0$, this reduces to $\ker\delta^\ast$: divergence-free edge flows representing conserved cycle currents that cannot be expressed as gradients of node potentials. In sheaves with nontrivial transport, holonomy around fundamental cycles introduces global compatibility obstructions, giving rise to geometric frustration.

\textbf{Heat dynamics and the analysis index.}
We state more details of the analysis index. The continuous-depth forward propagation of a diffusion-type GNN
can be written as
\[
\partial_t f=-\mathcal{L}_0f,
\qquad
f(t)=e^{-t\mathcal{L}_0}f(0),
\]
where $\mathcal{L}_0=\delta_{\mathcal F}^\ast\delta_{\mathcal F}$ is the $0$-cochain sheaf Laplacian. Since $\mathcal{L}_0$ is non-negative and self-adjoint, the heat operator satisfies $\lim_{t\to\infty}e^{-t\mathcal{L}_0}\to P_{\ker \mathcal{L}_0}$, where $P_{\ker \mathcal{L}_0}$ is the projection onto the harmonic space. Consequently, $\lim_{t\to\infty}\operatorname{Tr}(e^{-t\mathcal{L}_0})=\dim\ker \mathcal{L}_0$.
Thus the deep-limit survival space of node features is precisely the space
of global sections.

However, the heat trace $\operatorname{Tr}(e^{-t\mathcal{L}_0})$ itself is not a
topological invariant at finite depth. The topologically stable quantity is
the heat traceassociated with the Hodge--Dirac pair. Let $D_{\mathcal F}$ act on $\mathcal C^0(G;\mathcal F)\oplus \mathcal C^1(G;\mathcal F)$ defined by
\[
D_{\mathcal F}
=
\begin{pmatrix}
	0 & \delta_{\mathcal F}^\ast\\
	\delta_{\mathcal F} & 0
\end{pmatrix}
\]
with chiral components
\[
D_{\mathcal F}^{+}
=
\delta_{\mathcal F}:\mathcal C^0(G;\mathcal F)\to \mathcal C^1(G;\mathcal F),
\qquad
D_{\mathcal F}^{-}
=
\delta_{\mathcal F}^\ast:\mathcal C^1(G;\mathcal F)\to \mathcal C^0(G;\mathcal F).
\]
The analytic index is defined by
\[
\operatorname{ind}(D_{\mathcal F}^{+})
:=
\dim\ker D_{\mathcal F}^{+}
-
\dim\ker D_{\mathcal F}^{-}=\dim\ker\delta_{\mathcal F}
-
\dim\ker\delta_{\mathcal F}^{\ast}.
\]
By the discrete Hodge decomposition, this becomes $\operatorname{ind}(D_{\mathcal F}^{+})=\dim H^0(\mathcal F)-\dim H^1(\mathcal F)$.

\subsection{Proof of Proposition~\ref{prop:index_jump} (Topological Index and Index Jump)}
\label{proof.4}
\begin{proof}[Proof of Prop~\ref{prop:index_jump}]
	By the Euler--Poincar\'e formula applied to the sheaf cochain complex,
	\[
	\dim H^0(\mathcal F)-\dim H^1(\mathcal F)
	=
	\dim\mathcal C^0(G;\mathcal F)-\dim\mathcal C^1(G;\mathcal F).
	\]
	Since $\mathcal F$ has rank $r$,
	\[
	\dim\mathcal C^0(G;\mathcal F)=r|V|,
	\qquad
	\dim\mathcal C^1(G;\mathcal F)=r|E|.
	\]
	Therefore
	\[
	\operatorname{ind}(D_{\mathcal F}^{+})
	=
	r|V|-r|E|
	=
	r\chi(G).
	\]
	The same argument for the rank-$n$ Euclidean sheaf $\xi$ gives
	\[
	\operatorname{ind}(D_{\xi}^{+})=n\chi(G).
	\]
	Subtracting the two index identities gives
	\[\operatorname{ind}(D_{\mathcal F}^{+})-\operatorname{ind}(D_{\xi}^{+})
	=
	(r-n)\chi(G).
	\]
\end{proof}

\subsection{Details of Remark~\ref{remark2} (Index Jump Alone is Insufficient)}
The index comparison exposes a fundamental geometric tension:
\[
\dim H^0(\mathcal F)-\dim H^0(\xi)
=
\dim H^1(\mathcal F)-\dim H^1(\xi)
+
(r-n)\chi(G).
\]
For any connected graph with nontrivial cycle rank, one has $|E|\ge |V|$ and $\chi(G)\le 0$. Thus increasing the stalk dimension from $n$ to $r$ does not automatically
increase the harmonic capacity. When $r>n$ and $\chi(G)<0$, the rank term $(r-n)\chi(G)$ is strictly negative. Hence any strict increase of $H^0(\mathcal F)$ over
the Euclidean baseline must be supported by sufficient growth of the
frustration space $H^1(\mathcal F)$, which is interpreted as the divergence free space.

At the same time, the absolute dimension of the harmonic space should not be interpreted as a complete criterion for avoiding oversmoothing. One more promising approach is to focus on the harmonic space relation compares to simple baseline sheaf. More specifically, if there is a comparison map
\[
\iota_\ast:H^0(\xi)\longrightarrow H^0(\mathcal F),
\]
one should measure the additional non-dissipative capacity by the excess
quotient $H^0(\mathcal F)/\iota_\ast H^0(\xi)$, or at the level of dimensions, by $\dim H^0(\mathcal F)-\dim H^0(\xi)$. A positive excess dimension means that $\mathcal F$ contains harmonic modes
beyond the Euclidean constant-channel subspace. If these additional modes
are node-separating, they provide a structural mechanism against
oversmoothing.

Thus the index jump should be understood as a relative harmonic-capacity certificate, rather than as a pointwise oversmoothing theorem by itself. It does not merely count the size of a single harmonic space; it compares the
geometric sheaf against the Euclidean baseline and identifies when additional non-dissipative modes must appear. In the dimension-matched case $r=n$, the
rank penalty disappears and the comparison reduces to
\[
\dim H^0(\mathcal F)-\dim H^0(\xi)
=
\dim H^1(\mathcal F)-\dim H^1(\xi).
\]
Hence, in this case, growth of the frustration space $H^1(\mathcal F)$ is
exactly reflected as excess harmonic capacity beyond the Euclidean
oversmoothing baseline. In this sense, resistance to oversmoothing is
mediated by the growth of geometric frustration.

\subsection{Details of Theorem ~\ref{thm:main} (Genuine Harmonic Inclusion)} \label{app:thm5}
Theorem~\ref{thm:main} in the main text condenses the results developed in full detail throughout this appendix.
We investigate when the inclusion $\iota_\ast:H^0(\xi)\to H^0(\mathcal F)$ holds, beginning with the simplest case.
\begin{example}
    [Trivial holonomy as channel replication]
Assume that the sheaf connection has trivial holonomy. In this totally flat
regime, the rank-$r$ sheaf has {global trivialization} and decomposes into
$r$ independent scalar Euclidean channels:
\[
\mathcal F \cong \xi_{\mathrm{sc}}\otimes \mathbb R^r,
\]
where $\xi_{\mathrm{sc}}$ denotes the rank-$1$ scalar Euclidean sheaf with
stalk $\mathbb R$. Hence
\[
\dim H^k(\mathcal F)=r\beta_k(G),
\qquad
\beta_k(G):=\dim H^k(\xi_{\mathrm{sc}}).
\]
On the other hand, the Euclidean baseline $\xi$ has rank $n$ and satisfies
\[
\xi=\xi_{\mathrm{sc}}\otimes\mathbb R^n,
\qquad
\dim H^k(\xi)=n\beta_k(G).
\]
Therefore, one has $\dim H^0(\mathcal F)-\dim H^0(\xi)=(r-n)\beta_0(G)$. If $G$ is connected, then $\beta_0(G)=1$, and hence
\[
\dim H^0(\mathcal F)-\dim H^0(\xi)=r-n.
\]
Thus the totally flat case gives a strict capacity gain only by increasing
the number of trivial channels, i.e. only when $r>n$. In the
dimension-matched case $r=n$, trivial holonomy gives no excess harmonic
capacity over the Euclidean baseline.
\end{example}
\begin{remark}[Decoupling under trivial holonomy]
The decoupling follows from the existence of a global parallel frame. Fix a
base vertex $v_0\in V$. For each vertex $v\in V$, choose a path
$\gamma_{v_0v}$ from $v_0$ to $v$ and let
\[
Q_v:\mathcal F(v_0)\to\mathcal F(v)
\]
denote the parallel transport along $\gamma_{v_0v}$. Since the holonomy
around every closed cycle is trivial, $Q_v$ is independent of the chosen
path. Hence the maps $\{Q_v\}_{v\in V}$ give a global trivialization of the
sheaf. In this trivialization, every edge transport becomes the identity:
\[
Q_v^{-1}P_{vu}Q_u=I_r,
\]
for each oriented edge $u\to v$. Consequently, the coboundary operator
splits as $\delta_{\mathcal F}=\delta_{\xi_{\mathrm{sc}}}\otimes I_r$,
where $\xi_{\mathrm{sc}}$ is the rank-$1$ scalar Euclidean sheaf with stalk
$\mathbb R$. Thus
\[
\mathcal C^i(G;\mathcal F)
\cong
\mathcal C^i(G;\xi_{\mathrm{sc}})\otimes\mathbb R^r\qquad i=0,1.
\]

Taking cohomology gives $H^k(\mathcal F)
\cong
H^k(\xi_{\mathrm{sc}})\otimes\mathbb R^r$, and hence
\[
\dim H^k(\mathcal F)
=
r\dim H^k(\xi_{\mathrm{sc}}).
\]
Since $\xi=\xi_{\mathrm{sc}}\otimes\mathbb R^n$, one also has
\[
\dim H^k(\xi)
=
n\dim H^k(\xi_{\mathrm{sc}}).
\]
Therefore the totally flat regime does not create genuinely new geometric
cycle modes; it only replicates the scalar Euclidean cohomology in
independent channels. This motivates the non-trivial holonomy case, where
additional harmonic capacity must come from geometric frustration rather
than from trivial channel replication.
\end{remark}

\textbf{Non-trivial holonomy.} 
To make more intrinsic complexity of structure for anti-oversmoothing, one should not expect all cocycle-closed transports (holonomy) to be simultaneously trivialized. Instead, the Bochner--Weitzenb\"ock-type formula inspires to separate the cochains into two effect parts: the
rough discrepancy of edge-level modes and the curvature obstruction carried by the connection. We analysis the curvature influnced part, which is represented on the sheaf framework by the cocycle transport difference induced from restriction maps.

The inclusion of cochains means that part of the Euclidean
baseline is preserved inside the geometric sheaf. In representation-theoretic
terms, instead of requiring \(R=0\) or equivalently that the
stalk transport is trivial on the whole stalk, we require the restriction maps induces well-defined stalk transport, and the transport's holonomy representation has a nontrivial fixed
subspace. That is, there is a cocycle transport-fixed subspace
\[
W:=\bigcap_{\gamma\in\pi_1(G,v_0)}\ker(\rho(\gamma)-I)\subseteq\mathcal F(v_0),\qquad\dim W=k.
\]
\(W\neq \mathcal F(v_0)\) sites that the excess is not a full trivial-channel replication of the Euclidean baseline; the complementary directions carry nontrivial holonomy.

Every vector in \(W\) is fixed by parallel transport around all fundamental
cycles. Hence it generates a global parallel section of \(\mathcal F\). In
particular, when \(G\) is connected,
\[
H^0(\mathcal F)
\cong
\operatorname{Fix}(\rho),
\qquad
\dim H^0(\mathcal F)=k.
\]
Relative to the rank-\(n\) Euclidean baseline \(\xi\), this gives a strict increase of
\(H^0\) only when $k>\dim H^0(\xi)=n$ for connected \(G\). If one compares with the scalar Euclidean sheaf
\(\xi_{\mathrm{sc}}\), the corresponding condition is \(k>1\).

Recall in \cite{BGV} that the heat traceis
\[
\operatorname{Str}(e^{-tD_{\mathcal F}^2})
=
\operatorname{Tr}(e^{-tL^0_{\mathcal F}})
-
\operatorname{Tr}(e^{-tL^1_{\mathcal F}}).
\]
By the finite-dimensional Hodge decomposition, the nonzero spectra of \(L^0_{\mathcal F}\) and \(L^1_{\mathcal F}\) agree, so their nonzero heat contributions cancel in the supertrace. Hence $$\operatorname{Str}(e^{-tD_{\mathcal F}^2})
=
\dim H^0(\mathcal F)-\dim H^1(\mathcal F)
=
\operatorname{ind}(D^+_{\mathcal F})$$
which is just the interpretation of Atiyah-Singer index theorem combined with hodge theory. Equivalently, $\operatorname{Tr}(e^{-tL^0_{\mathcal F}})
=
\operatorname{ind}(D^+_{\mathcal F})
+
\operatorname{Tr}(e^{-tL^1_{\mathcal F}})$. Applying and subtracting this to \(\mathcal F\) and \(\xi\) gives
\[
\operatorname{Tr}(e^{-tL^0_{\mathcal F}})
-
\operatorname{Tr}(e^{-tL^0_{\xi}})
=
\Delta_{\mathrm{ind}}(\mathcal F,\xi)
+
\left[
\operatorname{Tr}(e^{-tL^1_{\mathcal F}})
-
\operatorname{Tr}(e^{-tL^1_{\xi}})
\right].
\]
Taking the large-time limit, using
\[
\lim_{t\to\infty}\operatorname{Tr}(e^{-tL^i_{\mathcal F}})
=
\dim H^i(\mathcal F),
\]
we obtain the corrected relative capacity identity, which is equivalent with the index jump equation Proposition~\ref{prop:index_jump}:
\[
\dim H^0(\mathcal F)-\dim H^0(\xi)
=
\Delta_{\mathrm{ind}}(\mathcal F,\xi)
+
\Delta_{\mathrm{Str}}^{(1)}(\mathcal F,\xi),
\]
where
\[
\Delta_{\mathrm{Str}}^{(1)}(\mathcal F,\xi)
:=
\lim_{t\to\infty}
\left[
\operatorname{Tr}(e^{-tL^1_{\mathcal F}})
-
\operatorname{Tr}(e^{-tL^1_{\xi}})
\right]
=
\dim H^1(\mathcal F)-\dim H^1(\xi).
\]
Thus the corrected capacity statement is not a replacement of the index idea, but a completion of it: the raw index gives the Euler-balanced part, while the degree-one heat trace supplies the missing first-cohomology contribution needed to recover \(H^0\).

For 1-cochains \(H^1\), the holonomy fixed decouple in 0-cochain case above requires a stronger condition, we call it together with previous condition of $W$ the "local system condition". Considering that a fixed subspace at the base stalk controls global sections, but it does not automatically imply that the first cohomology of
the scalar Euclidean sheaf injects into \(H^1(\mathcal F)\). To obtain a
decoupled \(H^1\)-contribution, we assume that
\begin{itemize}[leftmargin=2em]
\item  \(W\) extends to a
rank-\(k\) trivial sub-sheaf, i.e. the \(W\)-channel defines
a subcomplex
\[
    \mathcal C^i(G;\xi_{\mathrm{sc}})\otimes W
    \longrightarrow\mathcal C^i(G;\mathcal F),
    \qquad i=0,1
    \]
whose induced map on first cohomology is injective:
\[
H^1(\xi_{\mathrm{sc}})\otimes W
\hookrightarrow
H^1(\mathcal F).
\]
\end{itemize}
In the local-system case, the corrected index--heat trace balance reduces exactly to the holonomy-fixed excess:
\[
\dim H^0(\mathcal F)-\dim H^0(\xi)=k-n.
\]
This explains how nontrivial holonomy can compensate the negative raw index on cyclic graphs. When \(\beta_1(G)\) is large, the raw index jump $\Delta_{\mathrm{ind}}(\mathcal F,\xi)
=
(r-n)(1-\beta_1(G))$ may be negative, and the 1-cochain heat trace correction records the cycle-level geometric frustration and restores the correct zeroth harmonic capacity. Consequently, the strict capacity condition can be certified either by the corrected analyst balance $\Delta_{\mathrm{ind}}(\mathcal F,\xi)
+
\Delta_{\mathrm{Str}}^{(1)}(\mathcal F,\xi)
\ge 1$ or by the algebraic setting $\dim W>\dim H^0(\xi)$.

We can proceed with the derivation under local system. By analysing the decoupling, one has $\dim H^1(\mathcal F)
\ge
k\dim H^1(\xi_{\mathrm{sc}})$.
Since the rank-\(n\) Euclidean baseline satisfies
\[
\xi=\xi_{\mathrm{sc}}\otimes \mathbb R^n,
\qquad
\dim H^1(\xi)=n\dim H^1(\xi_{\mathrm{sc}}),
\]
we obtain the comparison lower bound
\[
\dim H^1(\mathcal F)-\dim H^1(\xi)
\ge
(k-n)\dim H^1(\xi_{\mathrm{sc}}).
\]
Equivalently, if \(G\) is connected and $\beta_1(G)=\dim H^1(\xi_{\mathrm{sc}})$, then
\[
\dim H^1(\mathcal F)-\dim H^1(\xi)
\ge
(k-n)\beta_1(G).
\]

Combining this with the index jump, we obtain the sufficient estimate
\[
\dim H^0(\mathcal F)-\dim H^0(\xi)
\ge
(k-n)\beta_1(G)+\Delta_{\mathrm{ind}}(\mathcal F,\xi)
\]
for connected \(G\). This reduces to $\dim H^0(\mathcal F)-\dim H^0(\xi)\ge(k-n)\beta_1(G)$, which gives the sufficient condition:
\begin{equation}\label{compare condition:little beta}
    \Delta_{\mathrm{ind}}(\mathcal F,\xi)\ge (n-k)\beta_1(G)+1
\end{equation}
This inspires an adjustment to increase anti-over smoothing capacity relative to the baseline,by adjusting restriction maps to make sufficiently \(k\) forces excess harmonic capacity over the Euclidean baseline.

Full trivial holonomy is the special case \(W=\mathcal F(v_0)\), so
\(k=r\). In that case the whole sheaf complex decouples into scalar
Euclidean channels. The present condition is weaker: the connection may have
nontrivial holonomy on the complementary directions, while a rank-\(k\)
holonomy-fixed channel still contributes Euclidean cohomology to the sheaf.
In this sense, the relevant condition is not global flatness of the entire
connection, but the existence of a holonomy-fixed subcomplex supporting
non-dissipative global and cycle-level modes.

One important note is that condition \eqref{compare condition:little beta} only fits the case that $\beta_1$ is not very large. Under local system condition, the equation is equivalent to $r-n-1\ge (r-k)\beta_1(G)$, which cannot be achieved when $\beta_1(G)$ is very large, since we cannot arbitrarily increase the rank of sheaf. One way to solve this problem is to replace the graph topological term $\beta_1(G)$ by some other intrinsic structure's topological term. For example, For graphs with well-defined geometric embedding properties, replacing the graph Betti number $\beta_1(G)$ with the first-order Betti number $\beta_1(M)$ of the embedding base space $M$ is a more reasonable choice, as it directly reflects the truly decisive topological information that might otherwise be obscured by the graph structure.
Constructing the replacement topological feature based on the graph and its intrinsic geometry that applies to more general scenarios, and deriving estimation inequalities with improved adaptability inspires one future work related here.

%A useful way to remove the purely rank-dependent Euler contribution is to introduce a reduced index. Define
%\[\operatorname{ind}_{\mathrm{red}}(\mathcal F;\xi):=\operatorname{ind}(D^+_{\mathcal F})-\frac{r}{n}\operatorname{ind}(D^+_{\xi}).\]
%For the rank-\(n\) Euclidean baseline,
%\[\operatorname{ind}(D^+_{\xi})=n\chi(G),\]
%and therefore
%\[\frac{r}{n}\operatorname{ind}(D^+_{\xi})=r\chi(G).\]
%Thus, for a pure rank-\(r\) trivial-channel replication, one has
%\[\operatorname{ind}_{\mathrm{red}}(\mathcal F;\xi)=0.\]
%This reduced index removes the background Euler-rank effect and separates genuine geometric contribution from trivial channel enlargement. However, by itself it still does not recover \(h^0\), because any index quantity remains an \(\dim H^0-\dim H^1\) invariant. To obtain the actual zeroth harmonic capacity, one must add back the degree-one correction.

\subsection{Details of the SPD Sheaf Case Study~\ref{sec:casespd}}\label{app:spd_verification}
We illustrate more details between SPDSheaf $\mathcal{F}_{\mathrm{SPD}}$~\cite{our} and our framework of applying Theorem~\ref{thm:main}. Recall that the SPD sheaf is equipped with stalks modeled on \(\mathrm{SPD}_n\), congruence-type restriction maps, an SPD coboundary operator, an admissible pairing, an adjoint, and Hodge Laplacians satisfying the harmonic/global-section correspondence. Through the logarithm map
\[
\log:\mathrm{SPD}_n\to\mathrm{Sym}_n,
\]
the SPD cochains are identified with symmetric-matrix-valued cochains of effective dimension
\[
r=\dim\mathrm{Sym}_n=\frac{n(n+1)}{2}.
\]
Under this identification, restriction maps of the form
\[
P\mapsto MPM^\top,\qquad M\in O(n),
\]
become linear isometric transports
\[
S\mapsto MSM^\top.
\]
Thus the SPD sheaf provides a concrete non-Euclidean instance where the finite-dimensional index-heat traceframework applies. We verify the required comparison ingredients as follows:

\begin{itemize}[leftmargin=2em]
    \item \textbf{Baseline-preserving embedding.}
    The Euclidean-to-SPD lifting
    \[
    \iota_\varepsilon(x)=xx^\top+\varepsilon I
    \]
    preserves Euclidean global sections at the harmonic-set level. Hence the Euclidean baseline is embedded into the SPD harmonic space:
    \[
    \iota_{\ast}H^0(\xi)\subseteq H^0(\mathcal F_{\mathrm{SPD}}).
    \]

    \item \textbf{Geometric transport.}
    The congruence-type transports generate nontrivial SPD-compatible parallel transport on \(\mathrm{Sym}_n\) after logarithmic linearization. The corresponding holonomy-fixed sector gives the geometric source of invariant harmonic directions, while nontrivial transport prevents the model from reducing to a full identity-sheaf tensor replication.

    \item \textbf{Index comparision.}
    The SPD Hodge complex satisfies the corrected relative capacity identity
    \[
    \Delta \dim H^0
    =
    \Delta_{\mathrm{ind}}
    +
    \Delta_{\mathrm{Str}}^{(1)}.
    \]
    Therefore the strict enlargement of SPD global sections should not be interpreted as a positive raw Euler-index jump on cyclic graphs. Instead, it is measured by the relative quotient
    \[
    H^0(\mathcal F_{\mathrm{SPD}})/\Phi_*H^0(\xi).
    \]
\end{itemize}

Since the SPD sheaf admits global sections beyond the lifted Euclidean image, the quotient
\[
H^0(\mathcal F_{\mathrm{SPD}})/\iota_{\ast}H^0(\xi)
\]
is nontrivial. This confirms that Theorem~\ref{thm:index_supertrace} detects genuine relative harmonic capacity, stronger than simply observing a larger ambient stalk dimension.

\begin{remark}
The purpose of this discussion is not to claim that the previous SPD sheaf
work already provides a complete anti-oversmoothing theory or experimental
validation. Rather, it shows that the structural ingredients required by the
index-comparison theorem can occur in a concrete non-Euclidean sheaf model.

First, the SPD sheaf provides geometric stalks modeled by
$\mathrm{SPD}_n$, together with restriction maps induced by
congruence-type transports and logarithmic SPD geometry [SPD sheaf,
Definition/Construction of SPD sheaf; Proposition on SPD restriction maps].
Together with an admissible pairing, these data define a coboundary operator
\[
\delta_{\mathcal F_{\mathrm{SPD}}}:
\mathcal C^0(G;\mathcal F_{\mathrm{SPD}})
\to
\mathcal C^1(G;\mathcal F_{\mathrm{SPD}}),
\]
its adjoint $\delta_{\mathcal F_{\mathrm{SPD}}}^{\ast}$, and the associated
Hodge Laplacian
\[
L_{0,\mathrm{SPD}}
=
\delta_{\mathcal F_{\mathrm{SPD}}}^{\ast}
\delta_{\mathcal F_{\mathrm{SPD}}}
\]
By Definition 4.7 in~\cite{our}, hence the SPD sheaf realizes
the Hodge-compatible operator package required by the general framework.

Second, the Euclidean-to-SPD embedding gives the baseline comparison map
\[
\iota:\xi\to\mathcal F_{\mathrm{SPD}}
\]

By the Euclidean-to-SPD embedding defined in Definition 4.10 of~\cite{our}, Euclidean global sections are preserved inside the SPD
sheaf, giving an induced injection
\[
\iota_\ast:
H^0(\xi)\hookrightarrow H^0(\mathcal F_{\mathrm{SPD}})
\]
By Proposition 4.13 in~\cite{our}, the SPD sheaf thus does not merely have a larger ambient stalk; it contains the Euclidean
harmonic baseline as a comparable subspace of global sections.

Third, the strict enlargement theorem for the SPD sheaf shows that
\[
\dim H^0(\mathcal F_{\mathrm{SPD}})
>
\dim H^0(\xi)
\]
By Theorem 4.15 in~\cite{our}-- or equivalently, since the Euclidean baseline
injects into the SPD global-section space, the quotient
\[
H^0(\mathcal F_{\mathrm{SPD}})
/
\iota_\ast H^0(\xi)
\]
is nontrivial. That is, this quotient measures the excess harmonic capacity beyond the Euclidean constant-channel baseline. Moreover, Appendix.B in~\cite{our} gives the index jump, and by calculating the discrete trace of $L_{\text{SPD}}$ one derives the certification $\Delta_{\text{ind}}+\Delta_{\text{ind}}>0$. In practice, it is enough for us to interpret that it holds Theorem~\ref{thm:main}.

This reinterpretation is important for two reasons. First, it shows that the
comparison theorem is not only formal: a concrete geometric sheaf can satisfy
the baseline-preserving injection and strict harmonic enlargement required
by the framework. Second, it separates the present contribution from the
previous SPD-specific construction. The earlier SPD sheaf work established a
strict global-section enlargement result, but it was not formulated as a
complete oversmoothing theory and did not provide anti-oversmoothing
experiments for SPD sheaves. The present paper extracts the general
comparison principle behind that result and later evaluates the
anti-oversmoothing behavior experimentally for SPD-type and nonlinear
variants.

Finally, the pairing and the index comparison play different roles. The
pairing determines the analytic energy
\[
\mathcal E_{\mathcal F}(x)
=
\langle
\delta_{\mathcal F}x,
\delta_{\mathcal F}x
\rangle,
\]
and hence the diffusion trajectory toward the harmonic space. The SPD prototype illustrates the central message of the linear theory:
anti-oversmoothing capacity should be compared relative to the Euclidean
baseline through an excess harmonic quotient, rather than read from the
absolute size of a single kernel.
\end{remark}

\section{Nonlinear Generalization}
\label{app:gyrosheaf}
\subsection{Details of Non-Linear Index Jump~\ref{nonlinearindexjump}}

Prior work assumes that the stalks $\mathcal F(v)$ are vector spaces, in which case the cochain spaces are:
\[
\mathcal C^0(G;\mathcal F)=\bigoplus_{v\in V}\mathcal F(v),
\qquad
\mathcal C^1(G;\mathcal F)=\bigoplus_{e\in E}\mathcal F(e)
\]
are linear spaces, or at least carry a compatible linear structure (e.g., a Lie group structure on SPD manifolds). The coboundary operator $\delta\colon\mathcal{C}^0(G;\mathcal{F})\to\mathcal{C}^1(G;\mathcal{F})$ is linear, and given an admissible inner product, one can define its adjoint $\delta^\ast$ and the $0$-Laplacian $\mathcal{L}_0=\delta^\ast\delta$, which is self-adjoint by construction. This linear structure is essential: it underpins the heat kernel, the McKean--Singer supertrace, and the resulting global index formula.

A fundamentally different setting arises when the stalks are nonlinear manifolds: each node feature is no longer a vector but a point on a manifold. The $0$-cochain space then becomes the product manifold
\[
\mathcal{C}^0 = \prod_{v \in V} M_v,
\]
where each $M_v$ is a (possibly distinct) nonlinear stalk manifold. Correspondingly, the coboundary operator generalizes to a nonlinear map $\delta\colon\mathcal{C}^0 \to \mathcal{C}^1$, with $\mathcal{C}^1$ denoting the space of edge constraints. The global section condition remains $\delta f = \mathbf{0}$, but the linear kernel $\ker\delta$ is now replaced by the \emph{nonlinear harmonic set}
\[
\mathcal{H} := \bigl\{ f \in \mathcal{C}^0 : \delta f = \mathbf{0} \bigr\}.
\]
When $\mathcal{H}$ is a smooth submanifold, it serves as the natural nonlinear analogue of the space of global sections.

In this nonlinear regime, the classical heat kernel machinery breaks down. While the resulting flow can still be interpreted as a nonlinear gradient flow, it is no longer generated by a linear trace-class operator on a fixed cochain vector space. As a result, the trace $\operatorname{Tr}(e^{-t\mathcal{L}_0})$ loses its spectral interpretation, and the McKean--Singer supertrace does not extend directly to this setting.

\textbf{Local tangent laplacian.} The correct replacement is local. Around a compatible state $f \in \mathcal{H}$, the compatibility map admits a differential
\[
d\delta_f \colon T_f\mathcal{C}^0 \longrightarrow T_{\mathbf{0}}\mathcal{C}^1,
\]
which serves as the tangent-space analogue of the coboundary operator. Given admissible inner products on the tangent spaces, we define its adjoint $(d\delta_f)^\ast \colon T_{\mathbf{0}}\mathcal{C}^1 \to T_f\mathcal{C}^0$ and the \emph{local tangent Laplacian}
\[
\mathcal{L}_{0,f} := (d\delta_f)^\ast\, d\delta_f.
\]
By construction, $\mathcal{L}_{0,f}$ is a self-adjoint linear operator on $T_f\mathcal{C}^0$. Although no global linear heat kernel exists in the nonlinear model, the local tangent heat operator $e^{-t\mathcal{L}_{0,f}}$ is well-defined, and its large-time limit projects onto
\[
\ker \mathcal{L}_{0,f} = \ker d\delta_f.
\]
That is, the surviving directions of the nonlinear diffusion are precisely those in the kernel of the linearized constraint.

When $\mathbf{0}$ is a regular value of $\delta$ near $f$, the implicit function theorem yields
\[
T_f\mathcal{H} = \ker d\delta_f, \qquad \dim T_f\mathcal{H} = \dim T_f\mathcal{C}^0 - \operatorname{rank}(d\delta_f).
\]
The nonlinear global-section capacity is therefore determined not by a single global kernel dimension, but by the local dimension of the harmonic submanifold.

\textbf{From global index to local analytic index.}
This also gives the correct nonlinear analogue of the index. The linearised operator $d\delta_f:T_f\mathcal{C}^0\to
T_{\mathbf{0}}\mathcal{C}^1$ induces a local analytic index:
\[
\operatorname{ind}_f(d\delta)
  :=\dim\ker d\delta_f-\dim\operatorname{coker}d\delta_f
  =\dim T_f\mathcal{C}^0-\dim T_{\mathbf{0}}\mathcal{C}^1,
\]
where the second equality holds because the spaces are finite-dimensional.
Hence \[
\dim T_f\mathcal{H}=\operatorname{ind}_f(d\delta)
+\dim\operatorname{coker}d\delta_f;
\]
If $d\delta_f$ is surjective, then $\dim T_f\mathcal{H}=\operatorname{ind}_f(d\delta)$, the local replacement of the global index formula.

\textbf{From global trace formula to local McKean--Singer cancellation.}
Define Laplacian
$\mathcal{L}_{1,f}:=d\delta_f(d\delta_f)^*$ on
$T_{\mathbf{0}}\mathcal{C}^1$. Since the nonzero spectra of $\mathcal{L}_{0,f}$ and
$\mathcal{L}_{1,f}$ coincide,
\[
  \operatorname{Tr}(e^{-t\mathcal{L}_{0,f}})-\operatorname{Tr}(e^{-t\mathcal{L}_{1,f}})
  =\dim\ker d\delta_f-\dim\ker(d\delta_f)^*
  =\operatorname{ind}_f(d\delta).
\]
Thus, the classical McKean–Singer cancellation survives locally after linearisation, even though the original nonlinear model carries no global trace formula.

\textbf{From global holonomy to local tangent holonomy.} When the stalks are manifolds, the differentials of the nonlinear restriction maps around $f\in\mathcal H$ define linear maps between tangent spaces, inducing a tangent sheaf:
\[
T_f\mathcal F(v):=T_{f_v}M_v.
\]
Along a closed cycle $\gamma$, the
composition of these tangent restriction maps gives a tangent holonomy
operator $d\rho_f(\gamma):T_{f_{v_0}}M_{v_0}\to T_{f_{v_0}}M_{v_0}$,
whose local parallel directions are controlled by
\[
  \operatorname{Fix}(d\rho_f)
  =\bigcap_{\gamma\in\pi_1(G,v_0)}\ker\bigl(d\rho_f(\gamma)-I\bigr).
\]
Thus, The nonlinear holonomy acts through the tangent representation and
controls $\ker d\delta_f$.

\subsection{The Calculation of \texorpdfstring{$D\delta$ and $(D\delta)^*$}{Ddelta and (Ddelta)*} of~\ref{differentil}}
\label{calculation}
For an oriented edge \(e=(u,v)\), the \emph{gyro-coboundary} is defined as:
\[
    \delta(X)_e = X_{ve} \ominus_{sym} X_{ue} = S_{ve}^{1/2} (X_{ve} - X_{ue}) S_{ue}^{1/2} 
\in\operatorname{Sym}_n=\mathcal F(e).
\] 

Fixing the $u$-side representative $X_{ue}$, we define the
\emph{$v$-side local gyro-coboundary map}:
\[
\delta_{ve}(\,\cdot\,)
:=
\delta(\,\cdot\,,X_{ue})
=
\,\cdot\, \ominus_{\mathrm{sym}} X_{ue},
\]
whose explicit form is:
\[
\delta_{ve}(X_{ve})
=
X_{ve}\ominus_{\mathrm{sym}}X_{ue}
=
(I-X_{ve}X_{ue})^{-1/2}(X_{ve}-X_{ue})(I-X_{ue}X_{ve})^{-1/2}.
\]
For notational brevity, set
\[
D_e:=X_{ve}-X_{ue},\qquad
L_e:=(I-X_{ve}X_{ue})^{-1/2},\qquad
R_e:=(I-X_{ue}X_{ve})^{-1/2}.
\]
Then
\[
\delta_{ve}(X_{ve}) = L_e D_e R_e .
\]

\paragraph{Tangent map.}
Let $\mathscr{S}_K[H] := D(K^{-1/2})[H]$ denote the Fr\'{e}chet
derivative of the principal matrix inverse square root at $K$ in the
direction $H$.
The tangent map of $\delta_{ve}$ at $X_{ve}$
\[
  d\delta_{ve}
  \;:\;
  T_{X_{ve}}\!\operatorname{Sym}(n)
  \;\longrightarrow\;
  T_{\delta_{ve}(X_{ve})}\!\operatorname{Sym}(n),
\]
is given by
\[
  d\delta_{ve}(P)
  \;=\;
  L_e\,P\,R_e
  \;+\;
  \mathcal{V}_{ve}(P),
  \qquad
  P \in T_{X_{ve}}\!\operatorname{Sym}(n),
\]
where
\[
 \mathcal{V}_{ve}(P)
  \;:=\;
  \mathscr{S}_{I-X_{ve}X_{ue}}\!\bigl[-P\,X_{ue}\bigr]\,D_e\,R_e
  \;+\;
  L_e\,D_e\;\mathscr{S}_{I-X_{ue}X_{ve}}\!\bigl[-X_{ue}\,P\bigr].
  \]
denote the Fréchet derivative of the principal matrix inverse square root.
Then the tangent map of $\delta_{ve}$ at $X_{ve}$ expressed as: 
\[
d\delta_{ve}: T_{X_{ve}} \mathrm{Sym}(n) \to T_{X_{ve} \ominus X_{ue}}\mathrm{Sym}(n)
\]

\[
d\delta_{ve}(P)
=
L_e P R_e
+
\mathcal V_{ve}(P),
\qquad
P\in T_{X_{ve}}\mathrm{Sym}(n),
\]
where
\[
\mathcal V_{ve}(P)
:=
\mathscr S_{\,I-X_{ve}X_{ue}}[-P X_{ue}]\,D_eR_e
+
L_eD_e\,\mathscr S_{\,I-X_{ue}X_{ve}}[-X_{ue}P].
\]

\paragraph{Frobenius adjoint.}
Let $\langle A,B\rangle_F := \operatorname{tr}(A^{\!\top} B)$ denote the
Frobenius inner product. The adjoint of~$d\delta_{ve}$ with respect
to this pairing is
\[
  (d\delta_{ve})^{\ast}(E)
  \;=\;
  \operatorname{sym}\!\Bigl(
    L_e^{\!\top}\,E\,R_e^{\!\top}
    \;+\;
    \mathcal{V}_{ve}^{\ast}(E)
  \Bigr),
  \qquad
  E \in T_{\delta_{ve}(X_{ve})}\!\operatorname{Sym}(n),
\]
where $\operatorname{sym}(Z) := \tfrac{1}{2}(Z + Z^{\!\top})$ and
\[
  \mathcal{V}_{ve}^{\ast}(E)
  \;:=\;
  -\,\mathscr{S}_{I-X_{ve}X_{ue}}^{\ast}\!
    \bigl(E\,R_e^{\!\top}\,D_e^{\!\top}\bigr)\,X_{ue}^{\!\top}
  \;-\;
  X_{ue}^{\!\top}\,
  \mathscr{S}_{I-X_{ue}X_{ve}}^{\ast}\!
    \bigl(D_e^{\!\top}\,L_e^{\!\top}\,E\bigr).
\]
Here $\mathscr{S}_K^{\ast}$ denotes the Frobenius adjoint of the
linear map $\mathscr{S}_K$.

\begin{remark}
While the closed-form expressions for $d\delta_{ve}$ and
$(d\delta_{ve})^{\ast}$ above may appear involved, we stress that they need \emph{not} be
implemented explicitly for gradient computation during training.
Because the gyro-Laplacian $\mathcal L$
satisfies
$(\mathcal L X)_v
:=
\nabla_{X_v}\mathcal E(X)=\bigl(D\delta(X)\bigr)^*D\delta(X)$, differentiating gyro-Laplacian end-to-end with automatic differentiation
yields exactly the same gradient as composing the tangent maps and
their adjoints derived here.
In all our experiments we therefore simply back-propagate through
$\Delta_{\mathrm{gyr}}$ directly, leveraging the \texttt{autograd} machinery of
PyTorch.
The explicit formulate are
presented for theoretical completeness: they make the gradient structure transparent.
\end{remark}

\subsection{Proof of Proposition~\ref{prop:Gyro-Valued Hodge} (Hodge-type Decomposition for the Nonlinear GyroSheaf)} \label{app:D3}
\begin{proof}
This construction is in the same spirit as the classical Hodge decomposition on a compact, oriented Riemannian manifold $M$ without boundary. In the continuous case, one starts from the space $A^k(M)$ of smooth $k$-forms. For any two such forms $\alpha,\beta\in A^k(M)$, the Riemannian metric together with the Hodge star operator $*$ determines an $L^2$-inner product:
\[
\langle \alpha,\beta\rangle_{L^2}
=
\int_M \alpha\wedge *\beta .
\]
Thus, after $L^2$-completion, $A^k(M)$ can be viewed inside a Hilbert space. The exterior derivative
\[
d:A^{k-1}(M)\to A^k(M)
\]
is a linear operator. Its adjoint
\[
d^*:A^k(M)\to A^{k-1}(M)
\]
is defined by
\[
\langle d\alpha,\beta\rangle_{L^2}
=
\langle \alpha,d^*\beta\rangle_{L^2}.
\]
Then the Hilbert space identity
\[
\operatorname{im}(d)
=
\ker(d^*)^\perp
\]
gives the decomposition
\[
A^k(M)
=
\operatorname{im}(d)
\oplus^\perp
\ker(d^*).
\]
Hence the direct-sum decomposition is obtained by combining three ingredients: a Hilbert structure, an inner product, and an adjoint operator.

In our discrete setting the same three ingredients are present. The role of the space of $k$-forms is played by the tangent spaces $T_X \mathcal{C}^0$ and $T_{\delta(X)} \mathcal{C}^1$, both of which are finite-dimensional inner-product spaces and hence, in particular, Hilbert spaces. The linearized boundary map
\[
D\delta_X : T_X \mathcal{C}^0 \longrightarrow T_{\delta(X)} C^1
\]
takes the place of the exterior derivative $d$, and its adjoint
\[
D\delta_X^* : T_{\delta(X)} \mathcal{C}^1 \longrightarrow T_X \mathcal{C}^0
\]
takes the place of the codifferential $d^*$. We now derive the corresponding complement relation. Observe that
\[
Y\in \operatorname{im}(D\delta_X)^\perp
\]
if and only if
\[
\langle D\delta_X h,Y\rangle=0,
\qquad
\forall\, h\in T_X\mathcal{C}^0.
\]
By the definition of the adjoint,
\[
\langle D\delta_X h,Y\rangle
=
\langle h,D\delta_X^*Y\rangle .
\]
Therefore,
\[
\langle D\delta_X h,Y\rangle=0,\quad \forall\, h
\]
if and only if
\[
D\delta_X^*Y=0.
\]
Hence
\[
\operatorname{im}(D\delta_X)^\perp
=
\ker(D\delta_X^*),
\]
which yields the decomposition
\[
T_{\delta(X)}\mathcal{C}^1
=
\operatorname{im}(D\delta_X)
\oplus^\perp
\ker(D\delta_X^*).
\]
Similarly, on the node-level tangent space, we obtain
\[
T_X\mathcal{C}^0
=
\operatorname{im}(D\delta_X^*)
\oplus^\perp
\ker(D\delta_X).
\]
\end{proof}

\section{Experimental Setup}
\label{app:experimental_setup}

This appendix details the implementation, datasets, hyperparameters, and runtime environment for
experiments reported in Section~\ref{sec:Experiments}. Two regimes are evaluated and configured separately: an untrained regime that isolates intrinsic propagation dynamics (Sec.~\ref{app:setup_untrained}), and a trained regime that verifies
these properties survive end-to-end optimization (Sec.~\ref{app:setup_trained}). Model implementations are described in
Sec.~\ref{app:setup_models}, and the compute environment details in Sec.~\ref{app:compute}.

\subsection{Datasets}
\label{app:setup_datasets}

\textbf{Datasets.} We use Cora, Citeseer and Texas since they span both homophilic citation graphs (Cora, Citeseer) and a heterophilic web graph (Texas), so that the layer-wise propagation behavior is evaluated under qualitatively different connectivity regimes. 

\subsection{Model Implementations}
\label{app:setup_models}

\textbf{Linear sheaf diffusion.} 
We use the original NSD discrete sheaf diffusion 
implementation~\cite{bodnar2022nsd} (the 
\texttt{neural-sheaf-diffusion} reference repository) for DiagSheaf, BundleSheaf and GeneralSheaf. IdentitySheaf instantiates the same NSD backbone and Laplacian builder as DiagSheaf, but the restriction maps are hard-coded to the 
identity matrix on every edge. For InverseSheaf, for each undirected edge $e=(i,j)$ we predict an invertible matrix $M_{e}$ from the concatenated
endpoint features through a single linear layer with a nonlinearity, parameterized via an LU factorization $M_{e} = L_{e} U_{e}$ with strict-triangular off-diagonals and softplus-rectified diagonals to ensure $M_{e}$ is invertible
throughout training. Edge transports in the both $i \to j$ and $j \to i$ direction use the same $M_{e}$.

\textbf{SPDSheaf.}
SPDSheaf is a from-scratch PyTorch implementation following the SPD-valued sheaf construction of~\cite{our}, simplified for the oversmoothing 
benchmark: we retain the core SPD geometry and sheaf Laplacian, but omit the molecule-specific dual-stream architecture and the coordinate-to-SPD 
lifting. Each node carries an SPD state 
$P_v \in \mathrm{SPD}_n$. 
% Initial SPD representations are obtained by projecting Euclidean node features through a single linear-plus-LeakyReLU layer to $\mathbb{R}^{d(d+1)/2}$, reshaping the output into a symmetric matrix $S \in \mathrm{Sym}_d$, and lifting via 
% $P_v^{(0)} = S S^\top + \varepsilon I_d$ to ensure positive definiteness.
For each edge $e=(i, j)$, the restriction map is parameterized by feeding the tangent-space vectorizations $[\mathrm{vec}(\log P_i) \,\|\, \mathrm{vec}(\log P_j)]$ through a single-layer local-concat sheaf learner~\cite{bodnar2022nsd}
to get an orthogonal matrix $M_{ie} \in O(n)$. 
% via the Cayley transform $M_{ue} = (I - S_{ue}/2)^{-1}(I + S_{ue}/2)$.
The restriction map acts by orthogonal congruence: $\mathcal{F}_{i \to e}(P) = M_{ie} P M_{ie}^\top$. Each layer applies the SPD sheaf Laplacian $\mathcal{L}_\mathrm{SPD}^{(\ell)}$ 
(Definition 4.7 of~\cite{our}), and then an
% assembled by transporting node states to the edge stalks via the restriction maps, taking tangent-space differences 
% $\log(\mathcal{F}_{v \to e} P_v) - \log(\mathcal{F}_{u \to e} P_u)$ at the 
% identity, and pulling the result back to incident vertices through 
% $M_{\cdot e}^\top (\cdot) M_{\cdot e}$ with appropriate edge-orientation 
% signs. Per-node eigenvalues of the resulting tangent-space update are 
% normalized to $[-1, 1]$ for numerical stability. 
Lie-group update is applied at the identity:
\[
  P_v^{(\ell+1)} 
  \;=\;  P_v^{(\ell)} \odot \mathcal{L}_\mathrm{SPD}^{(\ell)}( P_v^{(\ell)}).
\]
% \textbf{SPDSheaf.} 
% SPDSheaf is implemented from scratch in PyTorch following the 
% construction of~\cite{our}; we do not depend on an external 
% manifold-optimization library. Each node carries an SPD state 
% $P_v \in \mathrm{SPD}_n$. The restriction map for an edge $(u, v)$ is parameterized by passing the local feature pair through a single linear layer, applying a Cayley orthogonalization to produce $Q_{u\to e} \in O_n$, and acting by orthogonal congruence $P_v \mapsto Q_{u\to e}\, P_v\, Q_{u\to e}^\top$. Propagation is a Log--Euclidean tangent update of the form 
% $P \mapsto \mathrm{Exp}\!\big(\mathrm{Log}(P) + \mathcal{L}(P)\big)$, where $\mathrm{Exp}$ and $\mathrm{Log}$ are the matrix exponential 
% and logarithm at the identity, and $\mathcal{L}$ is the sheaf 
% Laplacian assembled from the restriction maps.

\textbf{GyroSheaf (ours).} 
GyroSheaf is also implemented from scratch in PyTorch. Each node carries an SPD state $P_v \in \mathrm{SPD}_n$, which is mapped to the symmetric unit ball $\mathcal{B}$ via the Cayley transform $\Psi(P) = (P - I)(P + I)^{-1}$ before propagation. Restriction maps are parameterized identically to SPDSheaf so that SPDSheaf and GyroSheaf differ only in how information is propagated across edges, isolating the effect of intrinsic curved propagation. The diffusion update is the raw Laplacian increment $X_{k+1} = X_k + \mathcal{L}(X_k)$ based on definition \ref{def:gyro_laplace}.

\subsection{Oversmoothing Measures}
\label{app:setup_measures}

% We adopt the normalized Dirichlet energy and mean-average-distance (MAD) as primary oversmoothing measures, with the Dirichlet energy form and MAD definition following \cite{rusch2023survey} with the $\lVert X \rVert_{F}^{2}$ normalization.

We adopt a Dirichlet energy measure and the mean-average-distance as primary oversmoothing measures. The pairwise summation form of both follows the survey of Rusch et al.\cite{rusch2023survey} (Eqs. (2) and (4) therein). For normalized Dirichlet energy, we additionally divide by $\lVert X \rVert_{F}^{2}$ to make the measure invariant to global feature rescaling, in line with~\cite{zhang2026measuring}.

% We adopt the normalized Dirichlet energy~\cite{cai2020note, zhou2021dirichlet} and mean-average-distance (MAD)~\cite{rusch2023survey} as primary oversmoothing measures. A central practical concern when comparing oversmoothing across heterogeneous propagation operators is that no single energy is operator-natural for all of them: Bison \& Sperduti \cite{deidda2025dirichlet} argue that measure-operator mismatch can produce misleading conclusions, implying that using each model's operator-natural energy. However, when the goal is direct cross-model comparison across qualitatively different operators (normalized GCN-style propagation, mean aggregation, sheaf Laplacians on linear stalks, curved propagation on SPD/gyro stalks), per-model operator-natural energies are no longer commensurable—each model is then measured by a different functional, conflating model differences with measurement differences. We therefore make the opposite trade-off: we adopt a uniform operator-agnostic $\ell_2$-based measure across all models, accepting that it is not natural to any single operator in exchange for cross-model commensurability.

% We use the normalized  Dirichlet energy and mean-average-distance (MAD)~\cite{rusch2023survey} as primary oversmoothing measures, evaluated for each model in its native output geometry with a common summation structure and normalization. 
% A third, 
% auxiliary measure---the normalized stable rank (NumRank)---is reported 
% only as a complementary diagnostic and is described separately in 
% Appendix~\ref{app:numrank}.

\textbf{Where the measure is evaluated.} 
All oversmoothing measures are computed on the \emph{backbone hidden representation immediately after the last propagation layer}, before any classification head. In the untrained regime this is the only representation produced by the model. In the trained regime this means we explicitly do not measure on the classifier logits, so that metrics are insensitive to head-specific design choices and remain directly comparable to the untrained trajectories.

\textbf{Neighborhood.} 
For all measures and all models, the neighborhood used in the 
pairwise summation is the symmetrized 1-hop neighborhood without self-loops,
\begin{equation*}
\mathcal{N}_i \;=\; \{\, j \in V \,:\, (i, j) \in E \,\},
\end{equation*}
applied identically across the Euclidean baselines, the linear 
sheaves, and the SPD-/gyro-valued models. 

% This convention is 
% independent of whether the underlying propagation operator was 
% augmented with self-loops (e.g.\ in GCN's $\hat{A}_{\mathrm{sym}}$); 
% the metric itself never sees a self-loop term.

\textbf{Euclidean and linear-sheaf models.} 
For models whose final-layer representation is a node feature matrix 
$X = \{x_v\}_{v \in V} \subset \mathbb{R}^{d}$ (GCN, GAT, GraphSAGE, DiagSheaf, BundleSheaf, GeneralSheaf, IdentitySheaf, InverseSheaf), both measures use the standard $\ell_{2}$ inner product on $\mathbb{R}^{d}$:
\begin{align*}
\mu_{\mathrm{NormDir}}(X) \;&=\; 
\frac{\tfrac{1}{|V|} \sum_{i \in V} \sum_{j \in \mathcal{N}_{i}} 
       \lVert x_{i} - x_{j} \rVert_{2}^{2}}
     {\sum_{v \in V} \lVert x_{v} \rVert_{2}^{2}}, \\[4pt]
\mu_{\mathrm{MAD}}(X) \;&=\; 
\frac{1}{|V|} \sum_{i \in V} \sum_{j \in \mathcal{N}_{i}}
\left( 1 - \frac{\langle x_{i},\, x_{j} \rangle}
                {\lVert x_{i} \rVert_{2} \, \lVert x_{j} \rVert_{2}} \right).
\end{align*}
% For DiagSheaf, BundleSheaf, GeneralSheaf, IdentitySheaf, and 
% InverseSheaf, the per-node feature vector $x_{v}$ is the full 
% $d \cdot d_{\mathrm{stalk}}$-dimensional vector obtained by 
% flattening the stalk-valued state at node $v$, so the metric is 
% computed on the entire sheaf-valued representation and not on a 
% per-stalk slice.

\textbf{SPD-valued models.} 
For models whose final-layer representation is SPD-valued, 
$P = \{P_{v}\}_{v \in V} \subset \mathrm{SPD}_n$, the natural 
distance is not the ambient $\ell_{2}$ distance but the 
Log-Euclidean distance, which is intrinsic to the SPD manifold. We therefore replace each $P_{v}$ by its image $S_{v} = \log(P_{v}) \in \mathrm{Sym}_n$, the $n \times n$ symmetric matrix, under the matrix logarithm and compute the corresponding measures with the Frobenius inner product on $\mathrm{Sym}_n$:
\begin{align*}
\mu^{\mathrm{SPD}}_{\mathrm{NormDir}}(P) \;&=\; 
\frac{\tfrac{1}{|V|} \sum_{i \in V} \sum_{j \in \mathcal{N}_{i}} 
       \lVert S_{i} - S_{j} \rVert_{F}^{2}}
     {\sum_{v \in V} \lVert S_{v} \rVert_{F}^{2}}, \\[4pt]
\mu^{\mathrm{SPD}}_{\mathrm{MAD}}(P) \;&=\; 
\frac{1}{|V|} \sum_{i \in V} \sum_{j \in \mathcal{N}_{i}}
\left( 1 - \frac{\mathrm{tr}(S_{i} S_{j})}
                {\lVert S_{i} \rVert_{F} \, \lVert S_{j} \rVert_{F}} \right).
\end{align*}
The Log-Euclidean measures admit an equivalent coordinate 
representation under any linear isomorphism $\mathrm{Sym}_n \to \mathbb{R}^{n(n+1)/2}$, but we formulate them at the matrix level to make the underlying Riemannian structure on $\mathrm{SPD}_n$ explicit and to avoid committing to an arbitrary coordinate chart.

\textbf{GyroSheaf.} 
GyroSheaf's final-layer representation lies in the symmetric unit ball $\mathcal{B} \subset \mathrm{Sym}_n$, equipped with a gyrovector-space geometry that is distinct from the Log-Euclidean structure used for SPDSheaf. To enable a controlled comparison between SPDSheaf and GyroSheaf, we deliberately evaluate both under a common metric: we recover an SPD representation via the inverse Cayley transform 
$
\Psi^{-1}(X) = (I + X)(I - X)^{-1},
$
and apply the SPD measures above to $\Psi^{-1}(X) \in \mathrm{SPD}_n$. This is an explicit trade-off: it sacrifices alignment with GyroSheaf's native gyro-geometry in exchange for commensurability with SPDSheaf, so that the SPDSheaf-vs-GyroSheaf comparison in {Appendix~\ref{subsec: gyrovsspd} isolates the effect of the propagation rule rather than confounding it with the choice of measurement geometry.

\textbf{Consistency across geometries.} 
All four measures share identical functional form and normalization; the only difference is the underlying inner product ($\ell_{2}$ on $\mathbb{R}^{d}$ versus Frobenius on $\mathrm{Sym}_n$). Since the feature dimensions are matched ($n(n+1)/2=120=d$) and all measures are scale-invariant, the trajectory shapes are directly comparable across all ten models.

% The summation structure, the normalization factor $1/|V|$, and the functional form of all four measures above are identical across the Euclidean and SPD cases; the only difference between 
% $\mu_{\mathrm{NormDir}}$ and $\mu^{\mathrm{SPD}}_{\mathrm{NormDir}}$ (and similarly for MAD) is the underlying inner product---$\ell_{2}$ on $\mathbb{R}^{d}$ versus Frobenius on $\mathrm{Sym}_n$---which matches to the linear-space representation of each model's output. Absolute numerical values are not directly comparable across the two geometries, since they
% measure pairwise dispersion in geometrically distinct spaces; what the $\lVert X \rVert_{F}^{2}$ normalization affords is scale invariance within each geometry, which renders each model's trajectory shape (monotonicity, decay rate, plateau behavior) invariant to global feature rescaling. It is this trajectory shape, not the absolute
% energy values, that we compare across all 10 models.

\textbf{Aggregation across seeds and depths.} 
For each $(\text{model}, \text{depth})$ configuration, the measures are computed independently for each of the five seeds and aggregated into a per-depth mean and min--max range. In all trajectory plots the solid curves show the per-seed mean and the shaded bands show the per-seed min--max range; we do not report standard deviations because at $5$ seeds the min--max range is a more honest summary of the variation actually observed.

\subsection{Untrained Regime}
\label{app:setup_untrained}

Following Rusch et al.~\cite{rusch2023survey}, both node features and model parameters are randomly initialized, and we record layer-wise representations under purely forward propagation; no optimizer, loss, or training step is involved. Depths $L \in \{1, 2, 4, 8, 16, 32, 64, 128\}$ are evaluated over five seeds $\{42, 43, 44, 45, 46\}$; trajectory plots report the mean and min--max range across seeds. All Euclidean and linear-sheaf models use a hidden width of $120$. For 
SPDSheaf and GyroSheaf, the SPD stalk dimension is $n = 15$, giving $n(n{+}1)/2 = 120$ independent parameters per node, matched to the Euclidean width. For all linear sheaf models from~\cite{bodnar2022nsd}, the auxiliary weights 
are disabled. These are feature-space transformations that do not enter the sheaf definition and disabling them isolates the intrinsic sheaf-geometric propagation characterized by Theorem~\ref{thm:main}. SPDSheaf and GyroSheaf do not have these auxiliary weights by construction.

% \textbf{Hidden dimension matching.} 
% All Euclidean and linear-sheaf models use a hidden width of $120$. For 
% SPDSheaf and GyroSheaf, the SPD stalk dimension is $n = 15$, giving 
% $n(n{+}1)/2 = 120$ independent parameters per node, matched to the 
% Euclidean width.

% \textbf{Initialization.} 
% The custom Euclidean baselines (GCN, GAT, GraphSAGE) and the input 
% projections of SPDSheaf/GyroSheaf use Xavier-uniform initialization with 
% the LeakyReLU gain. The local-concatenation sheaf learners (used by 
% SPDSheaf, GyroSheaf, IdentitySheaf, InverseSheaf) and the linear layers 
% inside the original NSD sheaf learners follow PyTorch's default 
% \texttt{nn.Linear} initialization. Following Bodnar et 
% al.~\cite{bodnar2022nsd}, the NSD backbone supports orthogonal 
% initialization on right weights and identity initialization on left 
% weights; in the figures reported in this paper these auxiliary weights 
% are disabled (\texttt{--no-left-weights --no-right-weights}), so this 
% choice does not affect the reported trajectories.

\subsection{Trained Regime}
\label{app:setup_trained}
We train every model end-to-end on Cora for $L \in \{2, 4, 8, 16, 32\}$, with five seeds $\{42, 43, 44, 45, 46\}$ per configuration. All oversmoothing curves are computed on the trained backbone hidden representations (not on 
the classifier logits) so that they are directly comparable to the untrained trajectories of App.~\ref{app:setup_untrained}. To isolate the effect of architecture on oversmoothing, we use a single shared training recipe across all model families, summarized in Table~\ref{tab:training_hparams}; all models are trained using cross-entropy loss with the label smoothing value listed therein, and early stopping is triggered when validation accuracy fails to improve for \texttt{Patience} consecutive epochs.

% To verify that the oversmoothing trends observed on Cora  generalize across graph topologies, we replicate the trained-regime study on Citeseer and Texas. Training protocol, seeds, depths, and hyperparameters follow Table~\ref{tab:training_hparams} unchanged; only the dataset is varied. Results are shown in appendix XX.

\begin{table}[h]
\centering
\small
\caption{Training hyperparameters (shared across all model families).}
\label{tab:training_hparams}
\begin{tabular}{lc}
\toprule
Hyperparameter & Value \\
\midrule
Learning rate                       & 0.004 \\
Weight decay                        & $10^{-4}$ \\
Epochs                              & 220 \\
Patience                            & 50 \\
Label smoothing                     & 0.05 \\
LR scheduler                        & ReduceLROnPlateau \\
%Dropout                             & 0.0 \\
Stalk dim. $d$ (sheaf models only)  & 8 \\
\bottomrule
\end{tabular}
\end{table}

\subsection{Compute Environment}\label{app:compute}
All experiments were conducted on a single NVIDIA GeForce RTX 3090 (24GB) GPU with Intel Xeon Gold 5218R CPUs and 376GB system memory. The total compute for all reported results, covering ten models, three datasets, five seeds, and eight depth configurations in both untrained and trained regimes, is modest and reproducible on a single consumer-grade GPU.

\section{Additional Experiments}
\label{app:additional_experiment}
\subsection{Within-Tier Discrimination: SPDSheaf vs. GyroSheaf under Nonlinear Distortion}
\label{subsec:gyrovsspd}\label{subsec: gyrovsspd}

\textbf{Motivation.}
Sections~\ref{sec:untrained} and~\ref{sec:trained} verify the \emph{between-tier} ordering predicted by Theorem~\ref{thm:main}. Within Tier~3, however, SPDSheaf and GyroSheaf share the same stalk dimension and restriction maps, therefore the same index jump $\Delta_\mathrm{ind}$ and holonomy fixed subspace $W$, so the index criterion alone cannot separate them. What actually distinguishes the two models is \emph{how each realizes the tangent complex} of Section~\ref{nonlinearindexjump}, and the goal of this experiment is to expose the regime in which this difference becomes observable.

\textbf{Two linearization strategies.}
SPDSheaf exploits the associative Lie-group structure of $\mathrm{SPD}_n$ to linearize the cochain complex through a single \emph{global chart}, namely the matrix logarithm (cf.\ Appendix~\ref{app:spd_verification}). The resulting linear Hodge complex falls within the finite-dimensional index--heat trace framework of Section~\ref{sec:main}. This shortcut is faithful when the observation map respects the global chart, but its fidelity degrades once the data departs from the globally-linearizable regime. GyroSheaf, by contrast, has no associative group operation on its gyrovector-space stalks (Section~\ref{gyrosheaf}); it therefore follows the \emph{pointwise} linearization program of Section~\ref{nonlinearindexjump}, building the tangent coboundary $d\delta_X$ and the tangent Hodge Laplacian afresh at each harmonic state $X$. This state-dependent construction makes no global-chart assumption and is expected to remain faithful when the observation manifold is nonlinearly distorted.

\textbf{Setup.}
Latent coordinates $\{z_v\} \subset \mathbb{R}^2$ are drawn from a 3-arm pinwheel distribution, and node labels are assigned by the angular coordinate of $z_v$ so that supervision is not linearly separable. The graph is the symmetrized 8-nearest-neighbor graph in latent space, held fixed independently of the observation map. Observed features $x_v \in \mathbb{R}^{16}$ are produced by a fixed nonlinear basis expansion of $z_v$ followed by a random linear projection. Splits are stratified by label with 20 train and 20 validation nodes per class.

To progressively distort the observation manifold, we inject isotropic Gaussian noise:
\[
  \tilde{x} \;=\; \frac{(x + \eta\,\varepsilon) - \mu}{\sigma},
  \qquad \varepsilon \sim \mathcal{N}(0, I),\quad \eta \in [0,\,0.8],
\]
where $(\mu, \sigma)$ are per-dimension statistics of the perturbed features. Re-standardization keeps scales comparable across noise levels. Latent coordinates, labels, and graph are held fixed; only $\eta$ varies.

\textbf{Quantifying global-chart breakdown.}
To compare the two strategies meaningfully, we need a model-agnostic measure of how far the observation map departs from local affinity, i.e.\ how stressed a fixed global chart should be. For each node $i$, let $\mathcal{N}_i$ denote its 8 latent-space nearest neighbors; fit a local affine map $(A_i, b_i)$ by least squares and define the normalized residual
\[
  \mathrm{AffErr}(X) \;=\; \frac{1}{|V|} \sum_{i \in V}
  \frac{\tfrac{1}{|\mathcal{N}_i|} \sum_{j \in \mathcal{N}_i}
  \|x_j - (A_i z_j + b_i)\|_2^2}{\mathrm{Var}(X)}.
\]
$\mathrm{AffErr}=0$ when the observation map is globally affine; larger values mark the breakdown of local-affine approximability. This is precisely the regime in which the fixed global-chart linearization (SPDSheaf) is expected to lose fidelity while the state-dependent tangent-complex construction (GyroSheaf) should remain stable.

\begin{figure}
    \centering
    \includegraphics[width=0.9\linewidth]{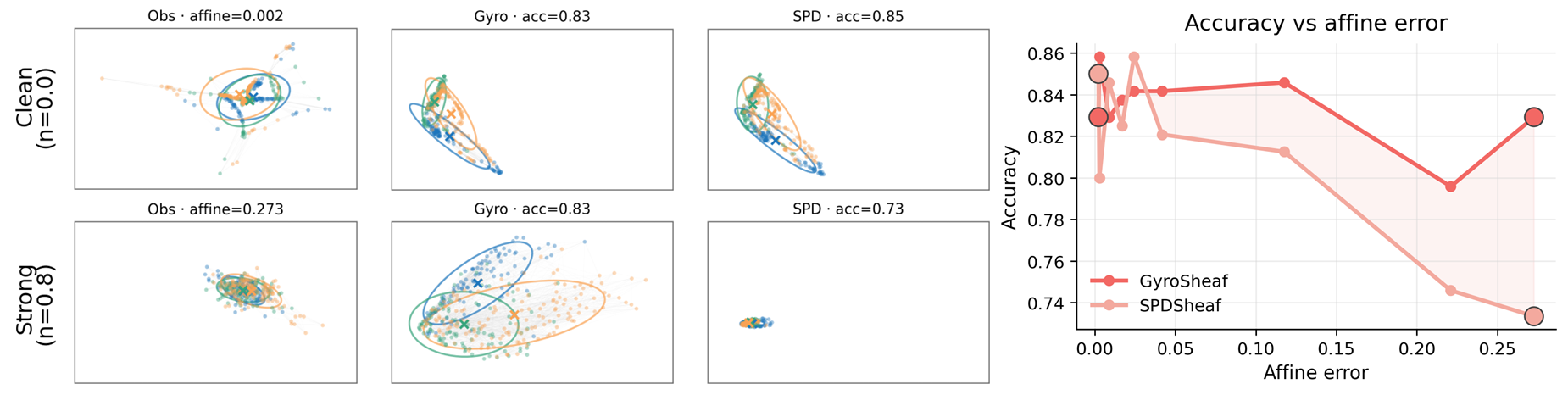}
    \caption{PCA projections of input features (left), GyroSheaf (middle), and SPDSheaf (right) at the clean ($\eta=0$, top) and strongly-perturbed ($\eta=0.8$, bottom) endpoints; ellipses show per-class covariance. Under strong perturbation, SPDSheaf collapses into a narrow band while GyroSheaf retains separable class structure. \emph{Right:}
  test accuracy versus $\mathrm{AffErr}(X)$. SPDSheaf degrades monotonically
  as $\mathrm{AffErr}$ grows; GyroSheaf remains flat.
  }
    \label{fig:gyrospd}
\end{figure}

\textbf{Results.}
Figure~\ref{fig:gyrospd} reports test accuracy against $\mathrm{AffErr}$. At $\mathrm{AffErr}=0.002$, where local linearity essentially holds, the two models result similar: when the global chart is faithful, both linearization strategies succeed. As $\mathrm{AffErr}$ grows to $0.273$, SPDSheaf trends downward to $0.73$, while GyroSheaf remains essentially flat at $0.83$. The two models become distinguishable here: GyroSheaf's pointwise tangent-complex construction tolerates nonlinear distortion that defeats SPDSheaf's global-chart linearization. In short, the index criterion governs harmonic capacity, while the linearization strategy provides a complementary axis of robustness once the data leaves the globally-linearizable regime.

\end{document}